\documentclass[sigconf,9pt,]{acmart}

\pdfoutput=1 

\usepackage{silence}
\usepackage{dutchcal}
\usepackage{amsmath}
\usepackage{amsfonts}
\usepackage{amsxtra}
\usepackage{amsthm} 
\usepackage{leftidx}
\usepackage{url}
\usepackage{graphicx}
\usepackage{float}
\usepackage{stfloats}
\usepackage{dsfont}
\usepackage{mathrsfs}
\usepackage{array}
\usepackage{rotating}
\usepackage{calc}
\usepackage{tikz-cd}
\usepackage{dsfont}
\usepackage{stmaryrd}
\usepackage{cancel}
\usepackage{accents}

\usepackage{enumitem}
\usepackage[linesnumbered,lined,boxed,commentsnumbered]{algorithm2e}
\SetAlFnt{\footnotesize}
\SetAlCapFnt{\rmfamily\footnotesize}

\SetNlSty{LineNumberFormat}{}{}

\SetCommentSty{mycommfont}
\usepackage[capitalize,nameinlink]{cleveref}
\usepackage{comment}
\usepackage{soul}
\usepackage{marginnote}

\usepackage{mathrsfs}

\usepackage{multicol}
\usepackage{subcaption}

\makeatletter
\newcommand\Func[2]{%
	\textbf{function} #1
	\algocf@group{#2}%
}
\makeatother

\makeatletter
\newcommand\Forr[2]{%
	\textbf{for} #1 \textbf{do}%
	\algocf@group{#2}%
}
\makeatother

\makeatletter
\newcommand\Blnk[2]{%
	\hspace{10pt}#1
	\algocf@group{#2}%
	\textbf{end}
}
\makeatother

\usepackage{color} 

\usepackage[normalem]{ulem}

\makeatletter
\newcommand{\removelatexerror}{\let\@latex@error\@gobble}
\makeatother

\makeatletter
\def\@endtheorem{\endtrivlist}
\makeatother

\newtheorem{theorem}{Theorem}
\newtheorem{definition}{Definition}
\newtheorem{proposition}{Proposition}
\newtheorem{remark}{Remark}
\newtheorem{corollary}{Corollary}

\newtheorem{problem}{Problem}
\newtheorem{subproblem}{Problem}[problem] 

\date{\today}

\newcommand{\nn}{{\mathscr{N}}\negthinspace}
\newcommand{\ou}{%
  \mathrel{%
	\vcenter{\offinterlineskip
	  \ialign{##\cr$<$\cr\noalign{\kern-1.5pt}$>$\cr}%
	}%
  }%
}%

\renewcommand{\marginnote}[1]{}
\newcommand{\overbar}[1]{\mkern 3.0mu\overline{\mkern-2.5mu#1\mkern-2.0mu}\mkern 2.0mu}

\newcommand\mathbftt[1]{\textnormal{\ttfamily\bfseries #1}}

\newcommand{\nnvec}{\nn_{\negthickspace f}}
\newcommand{\nnbf}{\nn_{\text{BF}}}
\newcommand{\nncomp}{\nn_{\negthickspace\text{BF}\circ\negthinspace f}}
\newcommand{\nnol}{\nn_{\negthickspace\text{O}}}
\newcommand{\nncontroller}{\nn_{\negthickspace\text{c}}}

\newcommand{\zsub}[1]{\mathcal{Z}_{\scriptscriptstyle \leq}(#1)}
\newcommand{\zsup}[1]{\mathcal{Z}_{\scriptscriptstyle \geq}(#1)}
\newcommand{\zeq}[1]{\mathcal{Z}_{\scriptscriptstyle =}(#1)}

\newcommand{\fliph}[1]{\mathfrak{F}(#1)}
\newcommand{\unfliph}[1]{\mathfrak{U}(#1)}
\newcommand{\flipset}[1]{\mathfrak{F}_{\scriptscriptstyle \{\negthinspace\cdot\negthinspace\}}\negthinspace(#1)}
\newcommand{\unflipset}[1]{\mathfrak{U}_{\scriptscriptstyle \{\negthinspace\cdot\negthinspace\}}\negthinspace(#1)}
\makeatletter
\newcommand{\acthyper}[2][\@nil]{
\def\tmp{#1}
\ifx\tmp\@nnil
	\ensuremath{\mathcal{a}}_{#2}
\else
	\ensuremath{\mathcal{a}}_{#1\thinspace|#2}
\fi}

\makeatletter
\newcommand{\removeonelatexerror}{\let\@latex@error\@gobble}
\makeatother

\copyrightyear{2025} 
\acmYear{2025} 
\setcopyright{rightsretained} 
\acmConference[ 2025]{ 2025}{~}{USA}
\acmBooktitle{28th ACM International Conference on Hybrid Systems: Computation and Control (HSCC '25), May XX--XX, 2025, XXX}\acmDOI{XXXXX}
\acmISBN{XXXXX}

\title{ %
Extracting Forward Invariant Sets from Neural Network-Based Control 
Barrier Functions %
} %

 \author{Goli Vaisi$^{*}$}
 \affiliation{
 	\institution{University of California, Irvine}
 	\department{Dept. of Electrical Engineering and Computer Science}
 	\city{Irvine}
 	\state{CA}
 	\country{USA}
 }
 \email{gvaisi@uci.edu}

 \author{James Ferlez$^{*}$}
 \affiliation{
 	\institution{University of California, Irvine}
 	\department{Dept. of Electrical Engineering and Computer Science}
 	\city{Irvine}
 	\state{CA}
 	\country{USA}
 }
 \email{jferlez@uci.edu}

 \author{Yasser Shoukry}
 \affiliation{
 	\institution{University of California, Irvine}
 	\department{Dept. of Electrical Engineering and Computer Science}
 	\city{Irvine}
 	\state{CA}
 	\country{USA}
 }
 \email{yshoukry@uci.edu}

\begin{document}


\begin{abstract}
Training Neural Networks (NNs) to serve as Barrier Functions (BFs) is a popular 
way to improve the safety of autonomous dynamical systems. Despite significant 
practical success, these methods are not generally guaranteed to produce true  
BFs in a provable sense, which undermines their intended use as safety 
certificates. In this paper, we consider the problem of formally certifying a 
learned NN as a BF with respect to state avoidance for an autonomous system: 
viz. computing a region of the state space on which the candidate NN is 
provably a BF. In particular, we propose a sound algorithm that 
efficiently produces such a certificate set for a shallow NN. Our algorithm 
combines two novel approaches: it first uses NN reachability tools to identify 
a subset of states for which the output of the NN does not increase along 
system trajectories;
then, it uses a novel enumeration algorithm for hyperplane arrangements to find 
the intersection of the NN's zero-sub-level set with the first set of states. 
In this way, our algorithm soundly finds a subset of states on which the NN  is 
certified as a BF. We further demonstrate the effectiveness of our algorithm at 
certifying for real-world NNs as BFs in two case studies. We complemented these 
with scalability experiments that demonstrate the efficiency of our 
algorithm.


\end{abstract}

\maketitle


\section{Introduction}
\label{sec:introduction}

Learning-enabled components, especially Neural Networks (NNs), have 
demonstrated incredible success at controlling autonomous systems. However, 
these components generally lack formal safety guarantees, which has inspired 
efforts to learn not just NN controllers, but also NN certificates of their 
safety. This approach has proven immensely successful at improving safety in 
practice, and at less computational cost than more rigorous methods. 
Unfortunately, learning safety certificates \emph{also} lacks formal 
guarantees, just as it does for learning controllers: i.e., attempts at 
learning safety certificates generally do not provide certificates that 
\emph{formally} assure safety. Nevertheless, the practical success of these 
methods suggests that learned safety certificates are good candidates for 
formal certification in  their own right.
\renewcommand*{\thefootnote}{} %
\footnote{$^{*}$Equally contributing authors.} %
\renewcommand*{\thefootnote}{\arabic{footnote}}
\setcounter{footnote}{0}

In this paper, we present an algorithm that can formally certify a NN as a 
Barrier Function (BF) for an autonomous, discrete-time dynamical systems. In 
particular, we propose a \emph{sound} algorithm that attempts to find a (safe) 
subset of the state space on which a given NN can be certified as a BF.  
Despite the overall goal of safety \emph{certification}, a sound algorithm is 
well-suited to this problem even though it is not guaranteed to return a safety 
certificate. On the one hand, a sound algorithm can be more efficient than a 
complete one, which complements the (relative) efficiency of learning  
certificates. On the other hand, the algorithm is intended to start with a NN 
that is already \emph{trained} to be a BF -- and hence it is likely that the NN 
can actually be certified as such; we show by case studies that this is indeed 
the case in practice. Hence, we propose an efficient algorithm that also likely 
to produce a safety certificate.

As a matter of computational efficiency, we base our algorithm on two 
structural assumptions, both of which facilitate more efficient BF 
certification. First, we assume that the learned BF candidate is a shallow 
Rectified Linear Unit (ReLU) NN. This assumption does not compromise the 
expressivity of the candidate NN 
\cite{HornikApproximationCapabilitiesMultilayer1991a}, but it implies the NN's 
linear regions are specified by a hyperplane arrangement (see  
\cref{sec:preliminaries}). As a result, we can leverage a novel and efficient 
algorithm for hyperplane arrangements (see 
\cref{sec:hyperplane_region_enumeration}). Second, we assume that the system 
dynamics are realized by a ReLU NN vector field; this implies that the 
(functional) composition of the candidate BF NN with the system dynamics is 
itself a ReLU NN. Hence, we can leverage state-of-the-art NN verification 
tools, such as {CROWN} \cite{ZhangEfficientNeuralNetwork2018}, to reason about 
this composition. Moreover, this assumption is motivated by the common use of 
ReLU NNs as controllers, which in turn inspired the choice of ReLU NNs to 
represent controlled vector fields (so the closed-loop system is also a ReLU 
NN).

Thus, our proposed algorithm takes as input a ReLU NN system dynamics, $\nnvec 
: \mathbb{R}^n \rightarrow \mathbb{R}^n$, a shallow NN trained as a BF, $\nnbf 
: \mathbb{R}^n \rightarrow \mathbb{R}$, and a set of safe states $X_s  \subset 
\mathbb{R}^n$; it then uses roughly the following two-step procedure to find a 
subset of the state space on which  $\nnbf$ can be certified as a BF for 
$\nnvec$.
\begin{enumerate}[left=0.1\parindent,label={\itshape (\roman*)}]
	\item Find a set $X_\partial \subseteq X_s$, on which $\nnbf$ decreases 
		along trajectories of $\nnvec$: i.e. for $x \in \mathcal{X}$, 
		$\nnbf(\nnvec(x)) - \gamma \nnbf(x) \leq 0$ for some $\gamma > 0$. By 
		assumption, $\nnbf \circ  \nnvec$ is a ReLU NN, so a NN 
		forward-reachability tool can be used to produce a set satisfying the 
		inequality above. See \cref{sec:forward_reachability}.

	\item Identify $X_c \subseteq X_\partial$, a connected component of 
		$\zsub{\nnbf} \triangleq \{x : \nnbf(x) \leq 0\}$, that lies entirely 
		within $X_c$ as provided by \emph{(i)}. This step entails reasoning 
		about the zero crossings of the shallow NN $\nnbf$, for which develop a 
		novel algorithm based on properties of hyperplane arrangements. See 
		\cref{sec:hyperplane_region_enumeration}.
\end{enumerate}
\noindent By the properties of a BF (and the additional condition \emph{(iv)} 
of \cref{prob:main_problem}), $\nnbf$ is certified as a BF on any set $X_c$  as 
above.




\noindent \textbf{Related work:} 
The most directly related works are 
\cite{WangSimultaneousSynthesisVerification2023,  
ZhangExactVerificationReLU2023, ChenVerificationAidedLearningNeural2024, 
ZhaoFormalSynthesisNeural2023}, though all but 
\cite{ChenVerificationAidedLearningNeural2024} consider continuous time 
systems. \cite{WangSimultaneousSynthesisVerification2023} certifies only the 
invariance of a safe set: it doesn't resolve which subset of safe states is 
actually invariant (see \cref{sec:hyperplane_region_enumeration}). 
\cite{ZhangExactVerificationReLU2023} attempts to find the zero-level set of a 
(continuous-time) barrier function, but it does so via exhaustive search with 
sound over-approximation. \cite{ChenVerificationAidedLearningNeural2024} 
consider ``vector barrier functions'', which are effectively affine 
combinations of ordinary barrier functions; 
\cite{ChenVerificationAidedLearningNeural2024} learns vector barrier functions 
by an iterative train-verify loop using NN verifiers for the usual barrier 
conditions. \cite{ZhaoFormalSynthesisNeural2023} considers polynomial dynamics 
and constraints, so the barrier properties are verified with an LMI.

By contrast, there is a wide literature on learning (Control) Barrier functions 
\cite{DawsonSafeControlLearned2022a, SoHowTrainYour2023}, but these works do 
not  formally verify their properties.  There is also a large literature on 
formal NN verification \cite{YangCorrectnessVerificationNeural2022, 
FerrariCompleteVerificationMultiNeuron2022, 
HenriksenDEEPSPLITEfficientSplitting2021,  
KhedrDeepBernNetsTamingComplexity2023, LiuAlgorithmsVerifyingDeep2021}, but 
none try to find the zero-level sets of NNs.

\section{Preliminaries} 
\label{sec:preliminaries}

\subsection{Notation} 
\label{sub:notation}
We will denote the real numbers by $\mathbb{R}$. For an $(n \times m)$ matrix 
(or vector), $A$, we will use the notation $\llbracket A \rrbracket_{[i,j]}$ to 
denote the element in the $i^\text{th}$ row and $j^\text{th}$ column of $A$.  
The notation $\llbracket A \rrbracket_{[i,:]}$ (resp. $\llbracket A 
\rrbracket_{[:, j]}$) will denote the $i^\text{th}$ row of $A$ (resp. 
$j^\text{th}$ column of $A$); when $A$ is a vector both notations will return a 
scalar.
$\lVert \cdot \rVert$ will refer to the max-norm on $\mathbb{R}^n$ unless 
noted, and  $\mathbftt{B}(x_0, \epsilon)$ as a ball of radius $\epsilon$ at 
$x_0$ (in $\lVert \cdot \rVert$ unless noted). For a set $S$, let $\overbar{S}$ 
denote its closure; let $\text{bd}(S)$ denote its boundary; let $\text{int}(S)$ 
denote its interior; and let $S^\text{C}$ denote its set complement. We will 
denote the cardinality of a finite set $S$ by $|S|$. For a function $f : 
\mathbb{R}^n \rightarrow  \mathbb{R}$, denote the \textbf{zero sub-level (resp. 
super-level)} set by  $\zsub{f} \triangleq \{ x | f(x) \leq 0 \}$ (resp. 
$\zsup{f} \triangleq \{ x | f(x) \geq 0 \}$); the \textbf{zero-crossing set} 
will be $\zeq{f}  \triangleq \{ x | f(x) = 0 \}$. Finally, let $f \circ g : x 
\mapsto f(g(x))$.


\subsection{Neural Networks} 
\label{sub:neural_networks}
We consider only  Rectified Linear Unit (ReLU) NNs. A $K$-layer ReLU NN is 
specified by $K$ \emph{layer} functions, which may be either linear or 
nonlinear. Both types are specified by parameters $\theta \triangleq (W,b)$ 
where $W$ is a  $(\overline{d} \times \underline{d})$ matrix and $b$ is a 
$(\overline{d} \times 1)$ vector. Then the \emph{linear} (resp.  
\emph{nonlinear}) layer given by $\theta$ is denoted $L_{\theta}$ (resp.  
$L_{\theta}^{\scriptscriptstyle \sharp}$), and is:
\begin{align}
	L_{\theta} &: \mathbb{R}^{\scriptscriptstyle \underline{d}} \rightarrow \mathbb{R}^{\scriptscriptstyle \overline{d}}, 
	   &L_{\theta} &:  z \mapsto Wz + b \\
	L_{\theta}^{\scriptscriptstyle \sharp} &: \mathbb{R}^{\scriptscriptstyle \underline{d}} \rightarrow \mathbb{R}^{\scriptscriptstyle \overline{d}},  
	    &L_{\theta}^{\scriptscriptstyle \sharp} &:  z \mapsto \max\{ L_{\theta}(z), 0 \}.
\end{align}
where $\max$ is element-wise. A $K$-layer ReLU NN is the functional composition 
of $K$ layer functions whose parameters $\theta^{|i}, i = 1, \dots, K$ satisfy 
$\underline{d}^{|i} = \overline{d}{\vphantom{d}}^{|i-1}: i = 2, \dots, K$; 
i.e., $\nn = L_{\theta^{|K}}^{\vphantom{{\scriptscriptstyle \sharp}}} \circ 
L_{\theta^{|K-1}}^{\scriptscriptstyle \sharp} \circ \dots \circ 
L_{\theta^{|1}}^{\scriptscriptstyle \sharp}$.

\begin{definition}[Shallow NN]
\label{def:shallow}
	A \textbf{shallow NN} has only two layers, with the second a linear layer: 
	i.e. $\nn_{s} = L_{\theta^{|2}}^{\vphantom{{\scriptscriptstyle \sharp}}} 
	\circ  L_{\theta^{|1}}^{\scriptscriptstyle \sharp}$.
\end{definition}

\begin{definition}[Local Linear Function]
\label{def:local_linear}
	Let $\nn: \mathbb{R}^n \rightarrow \mathbb{R}^m$ be a NN. Then an affine 
	function $\ell : \mathbb{R}^n \rightarrow \mathbb{R}^m$ is said to 
	be a \textbf{local linear (affine) function of} $\nn$ if $\exists$ $x_0\in  
	\mathbb{R}^n$ and $\epsilon > 0$ such that $\forall x \in  
	\mathbftt{B}(x_0,\epsilon) \; . \; \ell(x) = \nn(x)$.
\end{definition}


\subsection{Forward Invariance and Barrier Certificates}


\noindent The Theorem below describes sufficient conditions for a function such 
that it ensures a closed set is forward invariant.

\begin{theorem}[Barrier Function]
\label{thm:barrier_function}
	Consider a discrete-time dynamical system with dynamics $x_{t+1} = f(x_t)$, 
	where $x_t \in \mathbb{R}^n$. Suppose there is a $B: \mathbb{R}^n 
	\rightarrow \mathbb{R}$ and $\gamma \geq 0$ such that:
	\begin{align}
	\label{eq:deriv_condition}
		&B(f(x)) - \gamma B(x) \le 0, ~ \forall x \in \zsub{B}. 
	\end{align}
	Then $\zsub{B}$ is fwd. invariant and $B$ is a \textbf{barrier function}. 
\end{theorem}

\begin{remark}
	In practice, $B$ is chosen so that $\zsub{B}$ is strictly contained in some 
	problem-specific set of safe states, $X_s$
\end{remark}

\subsection{Hyperplanes and Hyperplane Arrangements} 
\label{sub:hyperplanes_and_hyperplane_arrangements}
Here  we review notation for hyperplanes and hyperplane arrangements.
\cite{EdelmanPartialOrderRegions1984} is the main reference for this section.
\begin{definition}[Hyperplanes and Half-spaces]
\label{def:hyperplane}
	Let $l : \mathbb{R}^n \rightarrow \mathbb{R}$ be an affine map. Then define:
	\begin{equation}
		H^{q}_{l} \triangleq 
		\begin{cases}
			\{x | l(x) < 0 \} & q = -1 \\
			\{x | l(x) > 0 \} & q = +1 \\
			\{x | l(x) = 0 \} & q = 0.
		\end{cases}
	\end{equation}
	We say $H^{0}_l$ is the \textbf{hyperplane defined by} $l$, and $H^{-1}_l$  
	(resp. $H^{+1}_l$) is the \textbf{negative (resp. pos.) half-space defined 
	by} $l$.
\end{definition}

\begin{definition}[Hyperplane Arrangement]
	Let $\mathcal{L}$ be a finite set of affine functions where each $l \in 
	\mathcal{L} : \mathbb{R}^n \rightarrow \mathbb{R}$. Then $\mathscr{H}  
	\triangleq \{ H^{0}_{l} | l \in \mathcal{L} \}$ is an \textbf{arrangement 
	of hyperplanes in dimension} $n$ . When $\mathscr{L}$ is important, we will 
	assume a fixed ordering for $\mathscr{L}$ via a bijection $\mathfrak{o} : 
	\mathscr{L} \rightarrow \{1, \dots, |\mathscr{L}|\}$, and also refer to 
	$(\mathscr{H},\mathscr{L})$ as a \textbf{hyperplane arrangement}.
\end{definition}

\begin{definition}[Region of a Hyperplane Arrangement]
\label{def:hyperplane_region}
	Let $(\mathcal{H},\mathscr{L})$ be an arrangement of $N$ hyperplanes in 
	dimension $n$. Then a non-empty subset $R \subseteq \mathbb{R}^n$ is said 
	to be a 
	\textbf{region of} $\mathcal{H}$ if there is an indexing function 
	$\mathfrak{s} : \mathcal{L} \rightarrow \{-1, 0, +1\}$ such that $R = 
	\bigcap_{l \in \mathcal{L}} H^{\mathfrak{s}(l)}_l$; $R$ is said to be 
	\textbf{full-dimensional} if it is non-empty and its indexing function 
	$\mathfrak{s}(l) \in \{-1, +1\}$ for all $l \in  \mathcal{L}$. Let 
	$\mathscr{R}$ be the set of all such regions of $(\mathscr{H},\mathscr{L})$.
\end{definition}

\begin{definition}[Face of a Region]
	\label{def:face}
	Let $\mathfrak{s}$ specify a full-dimensional region $R$ of a hyperplane 
	arrangement, $(\mathscr{H},\mathscr{L})$. A \textbf{face} $F$ of $R$ is a 
	non-empty region with indexing function $\mathfrak{s}^\prime$ s.t. 
	$\mathfrak{s}^\prime(\ell) = 0$ for all $\ell \in  \{\ell^\prime \in  
	\mathcal{L} | \mathfrak{s}^\prime(\ell) \neq \mathfrak{s}(\ell)\}$. $F$ is 
	\textbf{full-dimensional} if $\mathfrak{s}^\prime(\ell) = 0$ for exactly 
	one $\ell \in \mathcal{L}$.
\end{definition}

\begin{definition}[Flipped/Unflipped Hyperplanes of a Region]
\label{def:flips}
	Let $\mathfrak{s}$ specify a region $R$ of a hyperplane arrangement,  
	$(\mathscr{H},\mathscr{L})$. Then the \textbf{flipped hyperplanes of} $R$ 
	(resp. \textbf{unflipped}) are $\fliph{R} \triangleq \{ \ell \in 
	\mathscr{L} | \mathfrak{s}(\ell) > 0 \}$ (resp. $\unfliph{R} \triangleq \{  
	\ell \in \mathscr{L} | \mathfrak{s}(\ell) < 0 \}$). Further define 
	$\flipset{R} \triangleq \{ \mathfrak{o}(\ell)| \ell \in \fliph{R}\}$ and  
	$\unflipset{R} \triangleq \{ \mathfrak{o}(\ell)| \ell \in \unfliph{R}\}$.
\end{definition}

\begin{definition}[Base Region]
\label{def:base_region}
	Let $(\mathscr{H},\mathscr{L})$ be a hyperplane arrangement. A full 
	dimensional region  $R_b$ of $\mathscr{H}$ is the \textbf{base region of} 
	$\mathscr{H}$ if $| \unflipset{R_b} | = |\mathscr{L}| $ (and $\flipset{R_b} 
	= \emptyset$).
\end{definition}

\begin{proposition}
\label{prop:rebase}
	Let $(\mathscr{H},\mathscr{L})$ be a hyperplane arrangement. Then for any 
	region $R$ of $\mathscr{H}$, there are affine functions $\mathscr{L}_R$ 
	such that $(\mathscr{H},\mathscr{L}_R)$ is an arrangement with  base region 
	$R$.
\end{proposition}

\begin{proposition}
\label{prop:poset}
	Let $(\mathscr{H}, \mathscr{L})$ be a hyperplane arrangement with full  
	dimensional regions $\mathscr{R}$. Then the ordering $\leq$ on  
	$\mathscr{R}$:
	\begin{equation}
		R_1 \leq R_2 ~\text{iff}~ \flipset{R_1} \subseteq \flipset{R_2}
	\end{equation}
	makes $(\mathscr{R}, \leq)$ a poset, called the \textbf{region poset}.
\end{proposition}

\begin{proposition}[{\cite[Proposition 1.1]{EdelmanPartialOrderRegions1984}}]
	Let $(\mathscr{H}, \mathscr{L})$ be a hyperplane arrangement. Then its 
	region poset $(\mathscr{R}, \leq)$ is a \textbf{ranked poset} with rank 
	function $\text{rk}(R) = |\flipset{R}|$.
\end{proposition}

\begin{corollary}
\label{cor:region_successor}
	Let $(\mathscr{R},\leq)$ be the region poset of $(\mathscr{H},  
	\mathscr{L})$. If $R_2 \in \mathscr{R}$ covers $R_1 \in \mathscr{R}$, then 
	$\overbar{R}_1$ and $\overbar{R}_2$ are polytopes that share a 
	full-dimensional face (see \cref{def:face}).
\end{corollary}

\begin{corollary}
\label{cor:leveled_poset}
	The region poset $(\mathscr{R},\leq)$ can be partitioned into 
	\textbf{levels}, where level $k$ is $\mathscr{V}_k \triangleq \{R \in  
	\mathscr{R} : |\flipset{R}| = k\}$.
\end{corollary}

The following proposition connects local linear functions of a shallow NN to 
regions in a hyperplane arrangement.

\begin{proposition}
\label{prop:nn_local_linear} %
	Let $\nn$ be a shallow NN, and define its \textbf{activation boundaries} as:
	\begin{equation}
		\acthyper[\nn]{i} := x \mapsto \llbracket W^{|1} x + b^{|1} \rrbracket_{[i,:]}
	\end{equation}
	Now consider the hyperplane arrangement $(\mathscr{H}_{\nn}, 
	\mathscr{L}_{\nn})$, where  $\mathscr{L}_{\nn} \triangleq \{  
	\acthyper[\nn]{i} ~\big|~ i = 1, \dots , N \}$ and $\mathfrak{o}: 
	\acthyper[\nn]{i} \mapsto i$. (We will suppress the $\nn$ subscript when  
		there is no ambiguity.)

	Then let $R$ be any region (full-dimensional or not) of 
	$(\mathscr{H}_{\nn}, \mathscr{L}_{\nn})$ with indexing function 
	$\mathfrak{s}$. Then $\nn$ is an affine function on  $R$, and $\forall x 
	\in R ~.~ \nnbf(x) = \mathcal{T}^{\nn}_R(x)$ where
	\begin{equation}
	\label{eq:affine_region}
		{\mathcal{T}^{\nn}_R :
				x
				\mapsto ~
				W^{|2} \cdot 
				\left[
					\begin{smallmatrix}
					\tfrac{1}{2}(\mathfrak{s}(\acthyper{1}) + |\mathfrak{s}(\acthyper{1})|)\cdot\acthyper{1}\negthinspace(x) \vspace{-5pt} \\
					\vdots \\
					\tfrac{1}{2}(\mathfrak{s}(\acthyper{N}) + |\mathfrak{s}(\acthyper{N})|)\cdot\acthyper{N}\negthinspace(x)
					\end{smallmatrix} %
				\right] \negthickspace + b^{|2}.}
	\end{equation}
	That is, \eqref{eq:affine_region} nulls the neurons that are not active on 
	$R$.
\end{proposition}

\begin{remark}
\label{rem:shallow_vs_general_linear_functions}
	General ReLU NNs do not have hyperplane activation boundaries. Hence, 
	identifying their local linear functions is harder than for shallow NNs, 
	where hyperplane region enumeration  suffices.
\end{remark}




\section{Problem Formulation}
\label{sec:problem}

We now state the main problem of this paper: using a  candidate barrier 
function, $\nnbf$, we are interested in identifying a subset of a set of safe 
states, $X_s$, that is forward invariant for a given dynamical system. Thus, we 
certify $\nnbf$ as a BF on a subset of $X_s$.

\begin{problem}
\label{prob:main_problem}
	Let $x_{t+1} = \nnvec(x_t)$ be an autonomous, discrete-time dynamical 
	system where $\nnvec : \mathbb{R}^n \rightarrow \mathbb{R}^n$ is a ReLU  
	NN, and let $X_s \subset \mathbb{R}^n$ be a compact, polytopic set of safe  
	states. Also, let $\nnbf : \mathbb{R}^n \rightarrow \mathbb{R}$ be a 
	shallow ReLU NN (e.g. trained as a barrier function for $\nnvec$).

	Then the problem is to find a closed set $X_c \subseteq 
	X_s$ and $\gamma > 0$ s.t.:
	\begin{enumerate}[left=0.1\parindent,label={\itshape (\roman*)}]
		\item $\nnbf(\nnvec(x)) - \gamma \nnbf(x) \leq 0$ for all $x \in  
			X_c$;

		\item $X_c \subseteq \zsub{\nnbf}$;

		\item $\text{bd}(X_c) \subseteq \zeq{\nnbf}$; and

		\item $\nnbf(x) > 0$ for all $x \in \{ \nnvec(x^\prime) : x^\prime 
			\in X_c \} \backslash X_c$.

	\end{enumerate}
\end{problem}

Together, \emph{(i)}-\emph{(iii)} in \cref{prob:main_problem} match naturally 
with condition \eqref{eq:deriv_condition} of \cref{thm:barrier_function}. 
Indeed, in the special case where $X_c = \zsub{\nnbf} \subseteq  X_s$, 
conditions \emph{(i)}-\emph{(iv)} imply that \cref{thm:barrier_function} 
directly implies that  $X_c$ is forward invariant. Condition  \emph{(iv)} is 
redundant for  this case.

However, we are interested in a $\nnbf$ that is learned from data, so we can 
not assume that $\zsub{\nnbf} \subseteq X_s$. This presents an issue  
because of our discrete-time formulation: unlike in continuous-time, discrete 
time trajectories may ``jump'' from one connected component of $\zsub{\nnbf}$ 
to another. Thus, it is not enough to find a union of connected components of 
$\zsub{\nnbf}$ that are contained entirely in $X_s$, as is implied by 
conditions \emph{(ii)}-\emph{(iii)}. We must additionally ensure that no 
trajectories emanating from such a set can be ``kicked'' to another connected 
component of  $\zsub{\nnbf}$ by the dynamics $\nnvec$: hence, the need for 
condition \emph{(iv)}.

Thus, we have the following proposition, which formally justifies the 
conditions of \cref{prob:main_problem} with respect to our goal of obtaining a 
forward invariant subset of $X_s$.
\begin{proposition}
\label{prop:fi_prop}
	Let $\nnvec$, $\nnbf$ and $X_s$ be as in \cref{prob:main_problem}. Suppose 
	that there exists a closed set  $X_c \subseteq X_s$ and constant $\gamma 
	\geq 0$ such that conditions \emph{(i)}-\emph{(iv)} of 
	\cref{prob:main_problem} hold for $X_c$. Then the set $X_c$ is forward 
	invariant under $\nnvec$.
\end{proposition}
\begin{proof}
	Let $x_0 \in X_c$ be chosen arbitrarily. It suffices to show that 
	the point $x_1 = \nnvec(x_0) \in X_c$ as well.

	By assumption, $\nnbf(x_0) \leq 0$, and $\nnbf(x_1) - \gamma \nnbf(x_0)  
	\leq 0$. Thus, we conclude directly that $\nnbf(x_1) \leq 0$, and $x_1  \in 
	\zsub{\nnbf}$.

	Now we show that $x_1$ cannot belong to $\zsub{\nnbf} \backslash  
	X_c$. Suppose by contradiction that it does; then it follows from 
	condition \emph{(iv)} of \cref{prob:main_problem} that $\nnbf(x_1) > 0$, 
	which contradicts the above. Hence, $x_1 \in X_c$ necessarily, 
	and because we chose $x_0$ arbitrarily, we have shown that $X_c$ 
	is forward invariant.
\end{proof}
\begin{remark}
	It is trivial to construct examples of dynamics and BFs that satisfy all of 
	the conditions of \cref{thm:barrier_function}, but do not satisfy 
	\emph{(iv)} of \cref{prob:main_problem} for some choices of  
	$X_s$.
\end{remark}

The main difficulty in solving \cref{prob:main_problem} lies in the tension 
between condition \emph{(i)} on the one hand and conditions 
\emph{(ii)}-\emph{(iv)} on the other. Thus, we propose an algorithm that 
proceeds in a \emph{sequential} way: first attempting to identify where 
condition \emph{(i)} necessarily holds, and then within that set, where 
conditions \emph{(ii)}-\emph{(iv)} necessarily hold. These sub-algorithms are 
described by the following two sub-problems, \emph{both of which check 
\textbf{simpler} (sufficient) conditions for} \cref{prob:main_problem}. Note: 
only the first involves the system dynamics, $\nnvec$; the second is a property 
exclusively of $\nnbf$.

\begin{subproblem}
\label{prob:neg_deriv_prob}
	Let $\nnvec$, $\nnbf$ and $X_s$ be as in \cref{prob:main_problem}. Then the 
	problem is to identify a set  $X_\partial \subseteq X_s$ and a $\gamma \geq 
	0$ such that 
	\begin{equation}
		\label{eq:neg_deriv_condition}
		X_\partial \subseteq \zsub{\nnbf\circ\nnvec - \gamma \nnbf}.
	\end{equation}
\end{subproblem}

\begin{subproblem}
\label{prob:zero_set_prob}
	Let $\nnvec$, $\nnbf$ and $X_s$ be as in 
	\cref{prob:main_problem}, and let $X_\partial \subseteq  
	X_s$ be a solution for \cref{prob:neg_deriv_prob}. Then the 
	problem is to find $X_c \subset X_\partial$ such that:
	\begin{enumerate}[left=0.1\parindent,label={\itshape (\alph*)}]
		\item $X_c$ is a closed, connected component of $\zsub{\nnbf}$ with 
			$X_c = \overbar{\text{int}(X_c)}$ and $\text{bd}(X_c) \subseteq  
			\zeq{\nnbf}$; and

		\item $\nnbf(x) > 0$ for all $x \in C_{x_0} \backslash 
			X_c$ where $x_0 \in X_c$ and
	\end{enumerate}
\begin{equation}\label{eq:main_reach_bound}
	C_{x_0} \negthickspace\negthinspace \triangleq \negthinspace
	\mathbftt{B}\big(x_0,\negthinspace
	( \lVert \negthinspace \nnvec \rVert {\scriptstyle +}  1 \negthinspace) \cdot \negthickspace \sup_{x  \in X_c} \negthinspace 
	\lVert x {\scriptstyle -} x_0 \negthinspace \rVert {\scriptstyle +} \negthickspace \sup_{x  \in X_c} \negthinspace \lVert 
	\negthinspace \nnvec( x_0 \negthinspace ) {\scriptstyle -} x \rVert 
	\big).
\end{equation}
where $\lVert \nnvec \rVert$ is a bound on the Lipschitz constant of $\nnvec$.
\end{subproblem}

\cref{prob:neg_deriv_prob} is a more or less direct translation of condition 
\emph{(i)} in \cref{prob:main_problem}. However, a  solution to 
\cref{prob:zero_set_prob} implies conditions \emph{(ii)}-\emph{(iv)} in a less 
obvious way. In particular, condition \emph{(iv)} of \cref{prob:main_problem} 
is implied by condition \emph{(b)} of \cref{prob:zero_set_prob} by computing 
straightforward bounds on the reachable set $\nnvec(X_c)$.  Moreover, 
conditions \emph{(ii)}-\emph{(iii)} in \cref{prob:main_problem} are implied by 
condition \emph{(a)} in \cref{prob:zero_set_prob}, but the latter is easier to 
compute via hyperplane-arrangement algorithms, especially because of the 
insistence on a connected interior. The insistence on interior-connectedness is 
not particularly restrictive, since a solution to \cref{prob:zero_set_prob} can 
be applied multiple times to find distinct connected components. For ease of 
presentation, we defer these details to 
\cref{sec:hyperplane_region_enumeration}.

\begin{remark}
\label{rem:prob_differences}
	Choice of $\gamma$ aside, condition \eqref{eq:neg_deriv_condition} of 
	\cref{prob:neg_deriv_prob} is similar to 
	\cref{prob:zero_set_prob}\emph{(a)}. However, they differ in two other 
	important respects. First, unlike $\nnbf$, the function $\nnbf\circ\nnvec - 
	\gamma \nnbf$ is not a shallow network in our formulation. This means that 
	we cannot use the fast algorithm developed in 
	\cref{sec:hyperplane_region_enumeration} to solve 
	\cref{prob:neg_deriv_prob}. Second, in \cref{prob:zero_set_prob}, it is 
	important to find a set  $X_c$ that ``touches'' $\zeq{\nnbf}$;  this is 
	because of \cref{thm:barrier_function}. However, this is not necessary in 
	\cref{prob:neg_deriv_prob}, whose solution, $X_\partial$, can be relaxed 
	into the interior of $\zsub{\nnbf\circ\nnvec - \gamma \nnbf}$ as needed.
\end{remark}

\noindent \cref{sec:forward_reachability} presents our solution to 
\cref{prob:neg_deriv_prob}. \cref{sec:hyperplane_region_enumeration} presents 
our solution to \cref{prob:zero_set_prob}, and together these solve 
\cref{prob:main_problem}.

\ActivateWarningFilters[pdftoc]
\section{Forward Reachability of a NN to solve \cref{prob:neg_deriv_prob}} 
\label{sec:forward_reachability}
\DeactivateWarningFilters[pdftoc]

Solving \cref{prob:neg_deriv_prob} entails simultaneously resolving two 
intertwined challenges:
\begin{enumerate}[left=0.1\parindent,label={\itshape (\Alph*)}]
	\item identifying a single, valid $\gamma > 0$; and

	\item (under)approximating $\zsub{\nnbf\circ\nnvec - \gamma \nnbf}$ with 
		a set $X_\partial$.
\end{enumerate}
However, the fact that \emph{(B)} requires only an under-approximation of 
$\zsub{\nnbf\circ\nnvec - \gamma \nnbf}$ means that we can choose the members  
$x \in X_\partial$ based on sufficient conditions for $(\nnbf\circ\nnvec)(x) - 
\gamma \nnbf(x) \leq 0$ to hold. Indeed, given a test set $X_t$, the following 
Proposition provides a sufficient condition that $X_t \subseteq 
\zsub{\nnbf\circ\nnvec - \gamma \nnbf}$ for \emph{some} $\gamma >  0$; this 
condition is in turn based on lower and upper bounds of the functions $\nnbf 
\circ \nnvec$ and $\nnbf$.

\begin{proposition}\label{prop:partition_sufficient}
	Let $\nnbf$ and $\nnvec$ be as in \cref{prob:neg_deriv_prob}. Now let $X_t 
	\subseteq \mathbb{R}^n$, and suppose that for all $x \in X_t$, $l_f \leq  
	\nnbf(\nnvec(x)) \leq u_f$ and $l_\text{BF} \leq \nnbf(x) \leq u_\text{BF}$.

	Then $X_t \subseteq \zsub{\nnbf\circ\nnvec - \gamma \nnbf}$ if any of the 
	following hold (interpret division by zero as $\infty$):
	\begin{align}
		&u_f \negthinspace \leq \negthinspace 0 \thickspace \wedge \thickspace  l_\text{BF} \negthinspace \leq \negthinspace 0 \thickspace \wedge \thickspace 0 \negthinspace \leq \negthinspace \gamma \negthinspace \leq \negthinspace \frac{u_f}{l_\text{BF}} \label{eq:bound_1} \\
		&u_f \negthinspace \leq \negthinspace 0 \thickspace \wedge \thickspace  l_\text{BF} \negthinspace > \negthinspace 0 \thickspace \wedge \thickspace \gamma \negthinspace \geq \negthinspace 0 \label{eq:bound_2} \\
		&u_f \negthinspace \geq \negthinspace 0 \thickspace \wedge \thickspace  l_\text{BF} \negthinspace > \negthinspace 0 \thickspace \wedge \thickspace \gamma \negthinspace \geq \negthinspace \frac{u_f}{l_\text{BF}} \label{eq:bound_3}
	\end{align}
\end{proposition}
\begin{proof}
Consider condition \eqref{eq:bound_1}, and recall that $l_\text{BF} \leq 0$. 
Then for $x \in X_t$ and $l_\text{BF} < 0$ we have that:
\begin{equation}
	0 \leq \gamma \leq \frac{u_f}{l_\text{BF}} \implies \nncomp(x) \leq u_f = \frac{u_f}{l_\text{BF}} \cdot l_\text{BF} \leq \gamma \cdot l_\text{BF} \leq \gamma \nnbf(x). \notag
\end{equation}
When \eqref{eq:bound_1} holds with $l_\text{BF} = 0$ any $\gamma \geq 1$ will  
suffice, so choose $\gamma = 0$. The other conditions follow by similar 
arguments, noting $l_\text{BF} \geq 0$ in those cases (and the  special cases 
of $l_\text{BF} = 0$ in \eqref{eq:bound_2}).
\end{proof}

Note that a given choice of test set $X_t$ may fail to satisfy one of 
\eqref{eq:bound_1} - \eqref{eq:bound_3} for at least two reasons. The obvious 
reason is that there may not exist a $\gamma > 0$ that places the entirety of 
$X_t$ inside $\zsub{\nnbf\circ\nnvec - \gamma \nnbf}$. However, it may be the 
case that indeed $X_t \subseteq \zsub{\nnbf\circ\nnvec - \gamma \nnbf}$ for 
some $\gamma \geq 0$, but the bounds $u_f$ and $l_\text{BF}$ are too loose for 
the sufficient conditions in  \cref{prop:partition_sufficient} to be satisfied. 
Both possibilities suggest a strategy of recursively partitioning a test set 
$X_t$ until subsets are obtained that satisfy \cref{prop:partition_sufficient}. 
This allows finer identification of points that actually belong to 
$\zsub{\nnbf\circ\nnvec  - \gamma \nnbf}$, including by tightening the bounds 
$u_f$ and $l_\text{BF}$ (conveniently, NN forward reachability generally 
produces tighter results over smaller input sets; see 
\cref{sub:forward_reachability_and_linear_bounds_for_nns}).

However, such a partitioning scheme comes at the expense of introducing a 
number of distinct sets, each of which may satisfy the conditions of 
\cref{prop:partition_sufficient} for \emph{mutually incompatible bounds on} 
$\gamma$. For example, two such sets may satisfy \eqref{eq:bound_1} and 
\eqref{eq:bound_3} with non-overlapping conditions on $\gamma$. Fortunately, 
\eqref{eq:bound_2} and \eqref{eq:bound_3} share the common condition that 
$\nnbf(x) \geq 0$, which makes them essentially irrelevant for solving 
\cref{prob:zero_set_prob}; recall that \cref{prob:zero_set_prob} is interested 
primarily in subsets of $\zsub{\nnbf}$. Thus, we propose a partitioning scheme 
which partitions any set that fails \eqref{eq:bound_1} - \eqref{eq:bound_3}, 
but we include in $X_\partial$ only those sets that satisfy \eqref{eq:bound_1}. 
Given this choice, the minimum $\gamma$ among those sets satisfying 
\eqref{eq:bound_1} suffices as a choice of $\gamma$ for all of them.

We summarize this approach in \cref{alg:set_partition}, which contains a 
function \texttt{getFnLowerBd} for computing NN bounds (see 
\cref{sub:forward_reachability_and_linear_bounds_for_nns}). 
\cref{alg:set_partition} considers  only test sets that are hyperrectangles, in 
deference to the input requirements for \texttt{getFnBd}. Its correctness 
follows from the proposition below.

\begin{proposition}
Let $X_s$ be as in \cref{prob:neg_deriv_prob}, but suppose it is a 
hyperrectangle without loss of generality. Consider \cref{alg:set_partition}, 
and let $\mathcal{X} = \mathsf{getNegDSet}(X_s, \nnbf, \nnvec, \epsilon)$ with 
$X_\partial = \cup_{B \in \mathcal{X}} B$.

Then a nonempty $X_\partial$ so defined solves \cref{prob:neg_deriv_prob}.
\end{proposition}
\begin{proof}
	According to the construction of \cref{alg:set_partition}, a hyperrectangle 
	appears in $X_\partial$ if and only if it satisfies \eqref{eq:bound_1} of 
	\cref{prop:partition_sufficient}.

	Thus, it suffices to show that there exists a single $\gamma \geq 0$ such 
	that $X_\partial \subseteq \zsub{\nnbf\circ\nnvec - \gamma \nnbf}$. This 
	follows because $X_\partial$ is the union of finitely many hyperrectangles 
	$B \in \mathcal{X}$, each of which satisfies \eqref{eq:bound_1} for some 
	$\gamma_B > 0$. Thus $\gamma = \min_{B \in  \mathcal{X}} \gamma_B$ works 
	for $X_\partial$.
\end{proof}

\setlength{\textfloatsep}{0pt}
\IncMargin{0.5em}
%
%
\begin{figure}
 \vspace{4pt}
	\begin{minipage}[t]{\linewidth}
\begingroup %
\removeonelatexerror
\begin{algorithm}[H]
	\SetKw{Break}{break}
	\SetKw{NOT}{not}
	\SetKw{foriter}{for}
	\SetKw{OR}{or} 
	\SetKw{AND}{and} 
	\SetKw{IN}{in}
	\SetKw{IS}{is}
	\SetKw{CONT}{continue}
	\SetKw{GLOBAL}{global}

	\SetKwData{false}{False}
	\SetKwData{true}{True}
	\SetKwData{feas}{Feasible}

	\SetKwData{tab}{T}
	\SetKwData{relaxed}{relaxedConst}
	\SetKwData{succList}{successorList}
	\SetKwFunction{verSet}{VeifySet}
	\SetKwFunction{intpt}{FindInteriorPoint}
	\SetKwFunction{enum}{EnumerateRegions}
	\SetKwFunction{succ}{FindSuccessors}
	\SetKwFunction{intpt}{FindInteriorPoint}
	\SetKwFunction{len}{Length}
	\SetKwFunction{push}{push}
	\SetKwFunction{insertt}{insert}

	\SetKwFunction{negderiv}{getNegDSet}
	\SetKwFunction{ext}{getExtents}
	\SetKwFunction{apd}{append}
	\SetKwFunction{flower}{getFnBd}
	\SetKwFunction{fupper}{getFnBd}
	\SetKwFunction{partt}{part}
	\SetKwFunction{join}{listJoin}

	\SetKwData{bds}{bds}
	\SetKwData{ret}{ret}

	\SetKwInOut{Input}{Input}
	\SetKwInOut{Output}{Output}
	\Input{
		$X_t$, test set (assume hyperrectangle); \\
		$\nnbf : \mathbb{R}^n \rightarrow \mathbb{R}$, candidate barrier function; \\
		$\nnvec: \mathbb{R}^n \rightarrow \mathbb{R}^n$, NN vector field; \\
		$\epsilon > 0$, minimum partition size parameter.
	}

	\Output{
		$\mathcal{X}$, a list of hyperrectangles such that \\
		\hspace{12pt} $X_\partial \triangleq \cup_{B \in \mathcal{X}} B \subseteq \zsub{\nnbf\circ\nnvec - \gamma \nnbf} \cap X_t$
	}

	\SetKwProg{Fn}{function}{}{end}%

	\BlankLine

	\Fn{\negderiv{$X_t$, $\nnbf$, $\nnvec$, $\epsilon$}}{

		$l_\text{BF} \gets$ $\llbracket$\flower{$\nnbf$, $X_t$}$\rrbracket_{[1,1]}$ \tcp{lower bound}

		$u_f \gets$ $\llbracket$\fupper{$\nnbf \circ \nnvec$, $X_t$}$\rrbracket_{[1,2]}$ \tcp{upper bound}

		\If{$l_\text{BF} \leq 0$ \AND $u_f \leq 0$}{

			\Return [ $X_t$ ]

		}

		\bds $\gets$ \ext{$X_t$}





		\uIf{$l_\text{BF} \leq 0$ \AND $u_f > 0$ \AND $\max_{i = 1, \dots, n} |\llbracket \mathsf{bds} \rrbracket_{[i,2]} - \llbracket \mathsf{bds} \rrbracket_{[i,1]}| > \epsilon$}{

			\tcc{Partition $X_t$ in $2^n$ hyperrectangles and recurse:}

			\Return \join{ \negderiv{\partt{$X_t$, 1}, $\nnbf$, $\nnvec$, $\epsilon$}, $\dots$, \negderiv{\partt{$X_t$, $2^n$}, $\nnbf$, $\nnvec$, $\epsilon$} }

		}\Else{

			\Return [~] \tcp{$X_t$ too small or irrelevant; don't recurse}

		}

	}

	\BlankLine

	\BlankLine

	\tcc{Helper function to obtain a bounding box for a set}

	\Input{
		$X \subset \mathbb{R}^n$,
	}

	\Output{
		$E$, an $(n \times 2)$ matrix specifying the extent of $X$
	}
	
	\Fn{\ext{$X$}}{

		$E \gets \left[\begin{smallmatrix}0 & 0 & \dots 0\\ 0 & 0 & \dots 0 \end{smallmatrix}\right]^{\text{T}}$

		\For {$i = 1 \dots n$}{

			$\llbracket E \rrbracket_{[i,:]}$ =
				$[\min_{x \in X} \llbracket x \rrbracket_{[i,:]}, \max_{x \in X} \llbracket x \rrbracket_{[i,:]}]$

		}

		\Return $E$

	}
	
		   
		
	   
	\BlankLine

\caption{\small Recursive identification of $X_\partial$ for \cref{prob:neg_deriv_prob}}
\label{alg:set_partition}
\end{algorithm} 
\endgroup
	\end{minipage}
\end{figure}
\DecMargin{0.5em}


\subsection{Forward Reachability and Linear Bounds for NNs} 
\label{sub:forward_reachability_and_linear_bounds_for_nns}



To complete a solution to \cref{prob:neg_deriv_prob}, it remains to define the 
functions \texttt{getFnBd} in \cref{alg:set_partition}. For these, we use 
{CROWN} \cite{ZhangEfficientNeuralNetwork2018}, which efficiently computes 
linear bounds for the neural network's outputs using linear relaxations.

\begin{definition}[Linear Relaxation]
\label{def:relaxation}

Let \( f: \mathbb{R}^n \to \mathbb{R}^m \) and  \( X = \{x \subset \mathbb{R}^n 
| \underline{x} \leq x \leq \overline{x}\} \) be a hyper-rectangle. The linear 
approximation bounds of $f$ are $\overline{A} \ x + \overline{b}$ and 
$\underline{A} \ x + \underline{b}$ with 
$\overline{A}_{(f,X)},\underline{A}_{(f,X)} \in \mathbb{R}^{m \times n}$ and 
$\overline{b}_{(f,X)},\underline{b}_{(f,X)} \in \mathbb{R}^m$ such that 
$\underline{A}_{[i,:]} \ x + \underline{b}_{[i,:]} \leq f_i(x) \leq 
\overline{A}_{[i,:]} \ x + \overline{b}_{[i,:]}, \forall x \in X$, for each $i 
\in \{1, \dots, m\}$. 

For each output dimension, the upper and lower bounds of the function can be 
determined by solving the optimization problems: 
\begin{align}
\label{eq:relaxation}
\overline{f}_i = \max\limits_{x \in X} \overline{A}_i \ x + \overline{b}_i ,\quad \underline{f}_i = \min\limits_{x \in X} \underline{A}_i \ x + \underline{b}_i
\end{align}
\end{definition}

Computing upper and lower bounds of a NN using linear relaxations provided by 
CROWN is summarized in \cref{alg:reachable_set}, which formally defines the 
function \texttt{getFnBd} as used in \cref{alg:set_partition}.

\setlength{\textfloatsep}{0pt}
\IncMargin{0.5em}
%
%
\begin{figure}
 \vspace{4pt}
	\begin{minipage}[t]{\linewidth}
\begingroup %
\removeonelatexerror
\begin{algorithm}[H]
	\SetKw{Break}{break}
	\SetKw{NOT}{not}
	\SetKw{foriter}{for}
	\SetKw{OR}{or} 
	\SetKw{IN}{in}
	\SetKw{CONT}{continue}
	\SetKw{GLOBAL}{global}

	\SetKwInOut{Input}{Input}
	\SetKwInOut{Output}{Output}
		\SetKwProg{Fn}{function}{}{end}%
	\SetKwFunction{getfnbd}{getFnBd}
	\Input{
		$X \subset \mathbb{R}^n$, input set; \\
		$\nn: \mathbb{R}^n \rightarrow \mathbb{R}^m$, NN function to upper/lower bound
	}

	\Output{
		$E$, an $(m \times 2)$ matrix of lower/upper bounds for $\nn$ over $X$
	}


	\Fn{\getfnbd{$\nn$,$X$}}{

		$E \gets \left[\begin{smallmatrix}0 & 0 & \dots 0\\ 0 & 0 & \dots 0 \end{smallmatrix}\right]^{\text{T}}$

		\tcc{Compute linear relaxation of $\nn$ over $X$ using CROWN}
		$[\underline{A}, \overline{A},\underline{b}, \overline{b}] \leftarrow  \text{CROWN}(\nn, X) $
		
		\BlankLine

		\For {$i = 1 \dots m$}{

			$\llbracket E \rrbracket_{[i,:]}$ =
				$[\min_{x \in X} \llbracket \underline{A}x + \underline{b} \rrbracket_{[i,:]}, \max_{x \in X} \llbracket \overline{A}x + \overline{b} \rrbracket_{[i,:]}]$

		}

		\Return $E$
		

		
		
		
		
		


		
			
	
	}

 
\caption{NN Bound Computation using CROWN}
\label{alg:reachable_set}
\end{algorithm} 
\endgroup
	\end{minipage}
\end{figure}
\DecMargin{0.5em}


\ActivateWarningFilters[pdftoc]
\section{Efficient Hyperplane Region Enumeration to solve \cref{prob:zero_set_prob}} 
\label{sec:hyperplane_region_enumeration}
\DeactivateWarningFilters[pdftoc]


Solving \cref{prob:zero_set_prob} entails verifying two distinct properties of 
a set $X_c \subset X_\partial$. However, those properties implicate a common 
core algorithm: verifying a \emph{pointwise} property for a subset of $\nnbf$'s 
zero sub-level (or super-level) set that has a connected interior. Property 
\emph{(a)} concerns $X_c$ as a subset of $\nnbf$'s zero sub-level set; and 
property \emph{(b)} concerns the complement of $X_c$ as a subset of $\nnbf$'s  
zero super-level set. 
Crucially, it is possible to check both \emph{(a)} and \emph{(b)} 
\emph{pointwise} over their respective sub- and super-level sets, i.e. by 
exhaustively searching for a contradiction. For \emph{(a)}, this contradiction 
is a point in the interior of $X_c$ that is also not in the interior of $X_s$; 
and for  \emph{(b)}, this contradiction is a point $x^\prime \in C_{x_0} 
\backslash X_c$ for which $\nnbf(x^\prime) \leq 0$.

Thus, our algorithmic solution for \cref{prob:zero_set_prob} has two 
components: a zero sub(super-)level set identification algorithm; and the 
pointwise checks for properties \emph{(a)} and \emph{(b)}. The  zero  sub-level 
set algorithm, in \cref{sub:zero_sub_level_sets}, is the main contribution of 
this section. The pointwise checks for \emph{(a)} and \emph{(b)} are described 
in \cref{sub:checking_property_a} \& \cref{sub:checking_property_b}, 
respectively.

\subsection{Zero Sub-Level Sets by Hyperplane Region Enumeration} 
\label{sub:zero_sub_level_sets}

In order to identify the zero sub(super-)level sets of $\nnbf$, we leverage  
our assumption that $\nnbf$ is a shallow NN. In particular, a shallow NN has 
the following convenient characterization of its zero sub(super-)level sets in 
terms of regions of a hyperplane arrangement, which follows as a corollary of 
\cref{prop:nn_local_linear}.
\begin{corollary}
\label{cor:hyperplane_zero_sets}
	Let $\nn$ be a shallow NN. Then we have:
	\begin{align}
		\zeq{\nn} &= \bigcup_{R \in \mathscr{R}} R \cap H^{0}_{\mathcal{T}^{\nn}_R}; \notag \\
		\zsub{\nn} &= \zeq{\nn} \cup \bigcup_{R \in \mathscr{R}} R \cap H^{-1}_{\mathcal{T}^{\nn}_R};
		\text{~and} \notag \\
		\zsup{\nn} &= \zeq{\nn} \cup  \bigcup_{R \in \mathscr{R}} R \cap H^{+1}_{\mathcal{T}^{\nn}_R}
		\label{eq:zero_level_regions}
	\end{align}
	where $\mathscr{R}$ is the set of regions of $(\mathscr{H}_{\nn}, 
	\mathscr{L}_{\nn})$ as defined in \cref{prop:nn_local_linear}, and  
	$\mathcal{T}^{\nn}_R$ is as in \cref{prop:nn_local_linear}.
\end{corollary}

\noindent \cref{cor:hyperplane_zero_sets} directly implies that fast 
hyperplane-region enumeration algorithms can be used to identify the zero 
sub(super-)level set  of a shallow $\nnbf$. Indeed, one could identify the full 
zero  sub(super-)level set by enumerating all of the full-dimensional regions 
of $(\mathscr{H}_{\nn}, \mathscr{L}_{\nn})$, and testing the conditions of 
\eqref{eq:zero_level_regions} for each one.

However, for \cref{prob:zero_set_prob}, we are only interested in a 
\emph{connected component} of the zero sub(super-)level set. Thus, we structure 
our algorithm around \emph{incremental} region enumeration algorithms 
\cite{FerlezFastBATLLNN2022}, which have two important benefits for this 
purpose. First, they identify hyperplane regions in a connected fashion, which 
is ideal to identify connected components. Second, they identify valid regions 
incrementally, unlike other methods that must completely enumerate all regions 
before yielding even one valid  region\footnote{The known big-O-optimal 
algorithm is of this variety. 
\cite{EdelsbrunnerConstructingArrangementsLines1986}}.



\subsubsection{Incremental Hyperplane Region Enumeration} 
\label{ssub:iterative_hyperplane_region_enumeration}

These algorithms have the following basic structure: given a list of valid 
regions of the arrangement,  $\mathscr{V}$, identify all of their adjacent 
regions --- i.e. those \emph{connected} via a full-dimensional face with some 
region $R \in \mathscr{V}$ --- and then repeat the process on those adjacent 
regions that are unique and previously un-visited. This process continues until 
there are no un-visited regions left. Thus, incremental enumeration algorithms 
have two components, given a valid region $R$:
\begin{enumerate}[left=0.1\parindent,label={\itshape (\Roman*)}]
	\item identify the regions $\mathscr{A}_R = \{R^\prime \in \mathscr{R} | 
		R^\prime$ $\text{and}$ $R$ $\text{share}$ $\text{a}$ $\text{full-dim.}$ $\text{face}\}$; and

	\item keep track of which of $\mathscr{A}_R$ haven't been previously 
		visited (and are unique, when considering multiple regions at once).
\end{enumerate}

Step \emph{(II)} is the least onerous: one solution is to use a hash 
table\footnote{See e.g. \cite{FerlezFastBATLLNN2022}. But there are other  
methods, such as reverse search, which uses geometry to track whether a region 
has been/will be visited \cite{AvisReverseSearchEnumeration1996}.} that hashes 
each region $R$ according to its flips set $\flipset{R}$; recall that  
$\flipset{R}$ is a list of integers that uniquely identifies the region $R$ 
(see \cref{def:flips}). By contrast, step \emph{(I)} is computationally 
significant: it involves identifying which  hyperplanes contribute  
full-dimensional faces of the region (see \cref{def:face}). \footnote{This is 
also equivalent to computing a minimum Hyperplane Representation (HRep) for 
each region in the arrangement, since each region is an intersection of $N$ 
half-spaces and so is a convex polytope (see \cref{def:hyperplane_region}).  
Thus, the full-dimensional faces also correspond to hyperplanes that cannot be 
relaxed without changing the region: i.e. these hyperplanes can be identified 
by relaxing exactly one at a time, and testing whether the result admits a 
feasible point outside of the original region.} 

In particular, the full-dimensional faces of a (full-dimensional) region can be 
identified by testing the condition specified in \cref{def:face}. This test  
can be made on a hyperplane using a single Linear Program (LP) by 
introducing a slack variable as follows. 
\begin{proposition}
	\label{prop:face_slack}
	Let $R$ be a full-dimensional region of $(\mathscr{H}, \mathscr{L})$ with 
	indexing function $\mathfrak{s}$. Then $\ell^\prime \in \mathcal{L}$ 
	corresponds to a full-dimensional face of $R$ iff the following LP has a  
	solution with non-zero cost.
	\begin{align}
		\max_{x, x_s} x_s 
		\text{ s.t. } &\wedge_{\ell \neq \ell^\prime} \left( \mathfrak{s}(\ell)\cdot \ell(x) + x_s \leq 0 \right) \notag \\
		&\wedge (\ell^\prime(x) = 0) \wedge (x_s \geq 0)
	\end{align}
\end{proposition}
\noindent A naive approach performs this test for each of the hyperplanes for 
each region, which requires exactly $N$ LPs per region. However, 
\cref{cor:leveled_poset} suggests a more efficient approach. That is, start 
with the base region, $\mathscr{V}_0 = \{R_b\}$, and proceed \emph{level-wise} 
(see \cref{cor:leveled_poset}): at each level, $\mathscr{V}_k$, all members of 
$R^\prime \in \mathscr{V}_{k+1}$ will share a full-dimensional face among the 
hyperplanes in $\unflipset{R}$ for some $R  \in \mathscr{V}_k$; i.e., each  of 
the regions in $\mathscr{V}_{k+1}$ is obtained by ``flipping'' one of the 
unflipped hyperplanes of a region in $\mathscr{V}_k$. The correctness of this 
procedure follows from \cref{cor:region_successor}, and is summarized in 
\cref{alg:level_wise_region_enumeration}\footnote{The  
$\text{\rm\textsf{addConstr}}$ input is provided for future use.}. It is main 
algorithm we will modify to identifying zero sub(super-)level sets in the 
sequel.

\setlength{\textfloatsep}{0pt}
\IncMargin{0.5em}
%
%
\begin{figure}
 \vspace{4pt}
	\begin{minipage}[t]{\linewidth}
\begingroup %
\removeonelatexerror
\begin{algorithm}[H]
	\SetKw{Break}{break}
	\SetKw{NOT}{not}
	\SetKw{foriter}{for}
	\SetKw{OR}{or} 
	\SetKw{IN}{in}
	\SetKw{IS}{is}
	\SetKw{CONT}{continue}
	\SetKw{GLOBAL}{global}

	\SetKwData{false}{False}
	\SetKwData{true}{True}
	\SetKwData{feas}{Feasible}

	\SetKwData{tab}{T}
	\SetKwData{tabl}{table}
	\SetKwData{relaxed}{relaxedConst}
	\SetKwData{succList}{successorList}
	\SetKwData{sel}{sel}
	\SetKwData{addconstr}{addConstr}
	\SetKwData{adjtest}{testHypers}

	\SetKwFunction{enum}{EnumerateRegions}
	\SetKwFunction{succ}{FindSuccessors}
	\SetKwFunction{solvelp}{SolveLP}
	\SetKwFunction{optcost}{cost}
	\SetKwFunction{len}{Length}
	\SetKwFunction{push}{push}
	\SetKwFunction{insertt}{insert}
	\SetKwFunction{appendd}{append}

	\SetKwInOut{Input}{Input}
	\SetKwInOut{Output}{Output}
	\Input{
		$\mathscr{L}$, set of affine functions for arrangement $(\mathscr{H},\mathscr{L})$; and \\
		$\mathfrak{s}_0$, indexing function for a valid region $R_0 \in \mathscr{R}$. \\
	}

	\Output{
		\tab, hash table of indexing functions for all \\
		$\qquad$full-dimension regions of the arrangement.
	}

	\BlankLine

	\GLOBAL{\tab $\gets \{\}$}

	\BlankLine

	\SetKwProg{Fn}{function}{}{end}%
	\Fn{\enum{$\mathscr{L}$, $\mathfrak{s}_0$}}{

		\tcc{Assume $R_0$ (given by $\mathfrak{s}_0$) is the base region WOLG; see \cref{prop:rebase}}

		\tab $\gets \{\mathfrak{s}_0\}$ \tcp{Init. region hash table}

		$\mathscr{V} \gets [\mathfrak{s}_0]$ \tcp{Init. current level list}

		\BlankLine

		\While{\len{$\mathscr{V}$}$> 0$}{

			$\mathscr{V}^\prime \gets \{\}$

			\For{$\mathfrak{s} \in \mathscr{V}$}{

				$\mathscr{V}^\prime$.\appendd{\succ{$\mathscr{L}$, $\mathfrak{s}$}}







			}

			$\mathscr{V} \gets \mathscr{V}^\prime$

		}
		
		\Return \tab

	}

	\BlankLine

	\Input{
		$\mathscr{L}$, affine functions for hyperplane arrangement; \\
		$\mathfrak{s}$, indexing function for a valid region; \\
		\adjtest, a list of affine functions to test for \\
		$\qquad$adjacency (default value$\thinspace=\unflipset{R_\mathfrak{s}}$); \\
		$\text{\rm \textsf{addConstr}}$, a list of extra affine constraints \\
		$\qquad$(default value $\thinspace=\negthickspace\{\}$)
	}

	\Output{
		\succList, A list of region indexing functions \\
		$\qquad$adjacent to $\mathfrak{s}$ in the next higher region poset level
	}

	\Fn{\succ{$\mathscr{L}$, $\mathfrak{s}$, \adjtest$=\unflipset{R_\mathfrak{s}}$, $\text{\rm \textsf{addConstr}} = \{\}$}}{

		\succList $\gets \{\}$

		\tcc{Flip hyperplanes to get constraints for region $R_\mathfrak{s}$ given by $\mathfrak{s}$:}

		$A \gets \left[
			\begin{smallmatrix}
				\mathfrak{s}(\mathfrak{o}^{-1}(1)) \cdot \mathfrak{o}^{-1}(1)(x) & \dots & \mathfrak{s}(\mathfrak{o}^{-1}(N)) \cdot \mathfrak{o}^{-1}(N)(x)
			\end{smallmatrix}
		\right]^\text{T}$

		\BlankLine

		$\mathsf{sel} \gets [1 \dots 1] $ \tcp{Constraint selector (\texttt{len}=$N$)}

		\BlankLine

		\tcc{Loop over unflipped hyperplanes:}

		\For{$i \in$ \adjtest}{

			\BlankLine

			$\ell_r \gets \llbracket A \rrbracket_{[i,:]}$

			\BlankLine

			$\llbracket \mathsf{sel} \rrbracket_{[i,:]} \gets 0$ \tcp{Don't apply slack to $\ell_r$}


			\BlankLine

			\tcc{Check \cref{prop:face_slack} LP:}

			\If{\solvelp{ \\
			$\qquad[0\cdot1, \dots, 0\cdot n, 1]$, \\ $\qquad A(x) + x_s\cdot \mathsf{sel} \leq 0 ~\wedge~ \ell_r(x)\geq 0 ~\wedge~ x_s \geq 0$ \\ $\qquad \wedge_{\ell \in \text{\rm \textsf{addConstr}}} \thinspace ( \ell(x) + x_s \leq 0 )$ \\~}.\optcost{} $> 0$}{

				\tcc{$i$ is a full-dimensional face}

				\tcc{Region index after flipping $i^\text{th}$ hyperplane:}

				$\mathfrak{s}^\prime := \ell \in \mathscr{L} \mapsto \begin{cases}\mathfrak{s}(\ell) & \mathfrak{o}(\ell) \not= i \\  -\mathfrak{s}(\ell) & \mathfrak{o}(\ell) = i \end{cases}$

				\If{$\mathfrak{s}^\prime \not\in $ \tab}{

					\succList.\push{$\mathfrak{s}^\prime$}

					\tab.\insertt{$\mathfrak{s}^\prime$}

				}

			}

			$\llbracket \mathsf{sel} \rrbracket_{[i,:]} \gets 1$ \tcp{Undo selection}

		}

		\Return \succList

	}

\caption{\small$\texttt{Hyperplane Region Enumeration}$}
\label{alg:level_wise_region_enumeration}
\end{algorithm} 
\endgroup
	\end{minipage}
\end{figure}
\DecMargin{0.5em}
 


\subsubsection{Zero Sub-level Set Region Enumeration} 
\label{ssub:zero_sub_level_set_region_enumeration}

Given a hyperplane arrangement, \cref{alg:level_wise_region_enumeration} has 
the desirable properties of identifying connected regions (by exploring via 
shared full-dimensional faces) and incremental region identification (helpful 
when not all regions need be identified). To solve \cref{prob:zero_set_prob}, 
we develop an algorithm that has these properties --- but for regions of 
$(\mathscr{H}_{\nn}, \mathscr{L}_{\nn})$ that intersect $\zsub{\nn}$. That is, 
we modify \cref{alg:level_wise_region_enumeration} so that it:

\begin{itemize}
	\item identifies regions of $(\mathscr{H}_{\nn}, \mathscr{L}_{\nn})$ 
		that are mutually \textbf{connected through the interior of 
		}$\zsub{\nn}$; and

	\item terminates when no more such regions exist.
\end{itemize}
Each of these desired properties requires its own  modification of 
\cref{alg:level_wise_region_enumeration}, which we consider in order below.



First, we modify the way \cref{alg:level_wise_region_enumeration} identifies 
adjacent regions, so that two regions are only ``adjacent'' if they  share a 
(full-dimensional) face \emph{and} that face intersects 
$\text{int}(\zsub{\nn})$; thus, each newly identified region is connected to a 
region in the previous level through $\text{int}(\zsub{\nn})$. From 
\cref{prop:nn_local_linear}, the faces of a region $R$ that intersect  
$\text{int}(\zsub{\nn})$ are determined directly by the linear zero-crossing 
constraint on that region, viz. $\mathcal{T}^{\nn}_R$. Indeed, by continuity of 
$\nn$, the $\mathcal{T}^{\nn}_R$ should be added as an additional linear 
constraint to the LP in \cref{prop:face_slack}. We formalize this as follows.
\begin{proposition}
	\label{prop:zsub_adjacency}
	Let $R$ be a full-dimensional region of $(\mathscr{H}_{\nn},  
	\mathscr{L}_{\nn})$ with indexing function $\mathfrak{s}$. Then 
	$\ell^\prime \in \mathcal{L}_{\nn}$ corresponds to a full-dimensional face 
	of $R$  \textbf{that intersects }$\text{int}(\zsub{\nn})$ iff this LP is 
	feasible with non-zero cost:
	\begin{align}
		\max_{x, x_s} x_s 
		\text{ s.t. } &\wedge_{\ell \neq \ell^\prime} \left( \mathfrak{s}(\ell)\negthinspace \cdot \negthinspace \ell(x) + x_s \leq 0 \right) \wedge (\ell^\prime(x) = 0) \notag \\
		&\wedge (\mathcal{T}^{\nn}_R\negthinspace (x) + x_s \negthinspace \leq \negthinspace 0) \wedge (x_s \geq 0)
		\label{eq:int_feasible_point}
	\end{align}
\end{proposition}
\begin{proof}
	We prove the reverse direction first. Let $(x^*,x^*_s)$ be an optimal 
	solution to \eqref{eq:int_feasible_point} with $x^*_s > 0$. By 
	\cref{def:face},  $(x^*,x^*_s)$ belongs to a face of $R$ contained by 
	$\ell^\prime$, and likewise $(x^*,x^*_s)$ belongs to  
	$\text{int}\zsub{\nn}$ since $\mathcal{T}^{\nn}_R(x^*) \leq -x^*_s < 0$.

	In the other direction, there exists an $\hat{x} \in  
	\cap_{\ell\neq\ell^\prime} H_{\ell}^{\mathfrak{s}(\ell)} \cap  
	H_{\ell^\prime}^0 \cap H^{-1}_{\mathcal{T}^{\nn}_R}$ by definition. Assume 
	that $\mathcal{T}^{\nn}_R \in  \mathscr{L}_{\nn}$ for convenience, and let 
	$\hat{x}_{s,\ell} > 0$ be the slack for each constraint $\ell \neq 
	\ell^\prime$ at $\hat{x}$. Now observe that if we set  $\hat{x}_s =  
	\min_{\ell \neq \ell^\prime} x_{s,\ell}$ then  $(\hat{x},\hat{x}_s)$ is a 
	feasible point for \eqref{eq:int_feasible_point}. Hence, 
	\eqref{eq:int_feasible_point} is feasible and has optimal $x^*_s>0$.
\end{proof}
\noindent \cref{fig:backward_pass} illustrates this adjacency mechanism  (among 
other things). For example, region $R_0$ has adjacent regions $R_1$, $R_2$, 
$R_3$, $R_4$, $R_5$,  $R_6$ and $R_{13}$ according to \cref{prop:face_slack}. 
However, $R_0$ only has adjacent regions $R_1$ and $R_2$ according to 
\cref{prop:zsub_adjacency}, since hyperplanes $2$ and $3$ are the only ones to 
contain faces that intersect  $\text{int}(\zsub{\nn})$.

\cref{alg:level_wise_region_enumeration}, modified by 
\cref{prop:zsub_adjacency}, returns only regions that intersect 
$\text{int}(\zsub{\nn})$, but it is not guaranteed to identify \emph{all} such 
regions. In particular, \cref{prop:zsub_adjacency} ignores certain 
full-dimensional faces for adjacency purposes, and equivalently, prevents the 
associated hyperplanes from being ``flippable'' in certain regions. The effect 
is one of masking the associated connections in the region poset, always  
between a region in one level and a region in the immediate successor level. As 
a result, these ignored faces effectively mask (level-wise) monotonic paths to 
certain regions through the region poset.  This interacts with 
\cref{alg:level_wise_region_enumeration} can only ``flip'' hyperplanes but not 
``un-flip'' hyperplanes --- i.e., proceed only monotonically from lower to 
higher levels. The result is that some regions, even those intersecting  
$\text{int}(\zsub{\nn})$, can be rendered inaccessible if all of their direct 
paths to the base region are masked by \cref{prop:zsub_adjacency}. This 
situation is illustrated in \cref{fig:backward_pass}, which shows how the 
modified algorithm fails to identify region $R_4$. In the top pane of 
\cref{fig:backward_pass}, notice that $R_2$ is discovered from $R_0$ by 
flipping hyperplane 3, and $R_3$ is discovered from $R_2$ by flipping 
hyperplane 1; however, $R_4$ can only\footnote{Indeed, there is no other path 
to $R_4$ by only flipping hyperplanes.} be discovered from $R_3$ by  
\textbf{un-flipping} hyperplane 3. The bottom pane of \cref{fig:backward_pass} 
shows the associated region poset with connections grayed out when they are 
hidden by \cref{prop:zsub_adjacency}.

\begin{figure}[t]
\vspace{6pt}
\centering
\includegraphics[width=\linewidth,trim={0cm 0cm 0cm 0cm},clip]{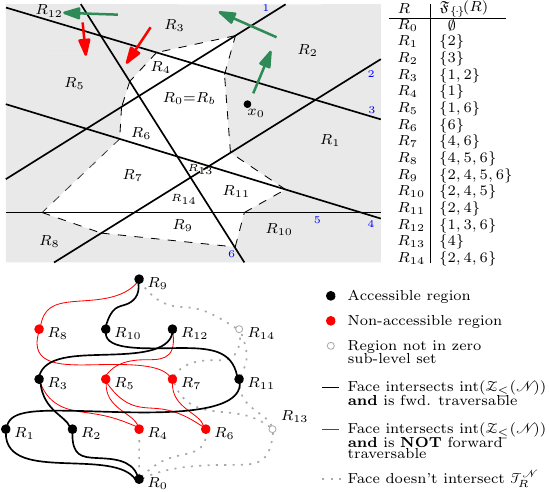}%
\caption{Illustration of the need for a backward pass to identify zero-sub-level sets of a shallow $\nn$. \textbf{Top: Shallow NN hyperplane arrangement.} The zero sub-level set is shaded gray; $R_0$ is the base region of the arrangement; hyperplane indices are shown in blue on the ``positive'' side; $\mathcal{T}^{\nn}_R(x)\negthickspace=\negthickspace 0$ hyperplanes are shown as dashed lines. A table shows $\flipset{R}$ for each labeled region \textbf{Bottom: Corresponding Region Poset (Partial).} Full dimensional regions are shown as nodes; full-dimension faces  as lines between nodes.}
\label{fig:backward_pass}
\end{figure}

Fortunately, \cref{fig:backward_pass} suggests a fix for the 
level-wise-increasing strategy of \cref{alg:level_wise_region_enumeration} --- 
without resorting to exhaustive region enumeration. In  
\cref{fig:backward_pass}, note that regions missed by the ``forward'' pass of 
\cref{alg:level_wise_region_enumeration} are nevertheless accessible by a 
``backward'' pass: i.e., \textbf{unflipping} a single hyperplane for a region 
discovered by the ``forward'' pass (these connections are highlighted in red in 
\cref{fig:backward_pass}). Thus, we propose an algorithm that generalizes this 
idea, and thereby ensures that all connected regions intersecting 
$\text{int}(\zsub{\nn})$ are visited. In particular, we propose 
\cref{alg:zero_sub_enumeration}, which replaces the function  
\texttt{FindSuccessors} of \cref{alg:level_wise_region_enumeration} to maintain 
(conceptually) separate ``forward'' and ``backward'' passes simultaneously. 
In particular, we initiate backward passes only through those faces of a 
region, $R$, which intersect both $H^0_{\mathcal{T}^{\nn}_R}$ \emph{and}  
$\text{int}(\zsub{\nn})$; this prevents backwards passes from being instigated 
on every unflipped hyperplane for each region.
Moreover, we employ two procedures to reduce the number regions from which 
backward passes are initiated. First, we precede each backward pass with a 
single LP that checks the region for \emph{any} intersection with its 
$H^0_{\mathcal{T}^{\nn}_R}$ hyperplane; second, we mark each region with the 
flipped or unflipped hyperplanes that discovered it, so these don't need to be 
unflipped or flipped again (omitted from \cref{alg:zero_sub_enumeration}).

\setlength{\textfloatsep}{0pt}
\IncMargin{0.5em}
%
%
\begin{figure}
 \vspace{4pt}
	\begin{minipage}[t]{\linewidth}
\begingroup %
\removeonelatexerror
\begin{algorithm}[H]
	\SetKw{Break}{break}
	\SetKw{NOT}{not}
	\SetKw{foriter}{for}
	\SetKw{OR}{or} 
	\SetKw{IN}{in}
	\SetKw{IS}{is}
	\SetKw{CONT}{continue}
	\SetKw{GLOBAL}{global}

	\SetKwData{false}{False}
	\SetKwData{true}{True}
	\SetKwData{feas}{Feasible}

	\SetKwData{tab}{T}
	\SetKwData{tabl}{table}
	\SetKwData{relaxed}{relaxedConst}
	\SetKwData{succList}{successorList}
	\SetKwData{predList}{predecessorList}
	\SetKwData{flips}{flips}
	\SetKwData{sel}{sel}
	\SetKwData{addconstr}{addConstr}
	\SetKwData{adjtest}{testHypers}

	\SetKwFunction{enum}{EnumerateRegions}
	\SetKwFunction{succ}{FindSuccessors}
	\SetKwFunction{succFwdBkwd}{FindSuccessorsFwdBkwd}
	\SetKwFunction{solvelp}{SolveLP}
	\SetKwFunction{optcost}{cost}
	\SetKwFunction{len}{Length}
	\SetKwFunction{push}{push}
	\SetKwFunction{insertt}{insert}
	\SetKwFunction{appendd}{append}

	\SetKwInOut{Input}{Input}
	\SetKwInOut{Output}{Output}
	\Input{
		$\mathscr{L}_{\nn}$, set of affine functions for a shallow NN arrangement $(\mathscr{H}_{\nn},\mathscr{L}_{\nn})$; and \\
		$\mathfrak{s}_0$, indexing function for a valid region $R_0 \in \mathscr{R}$. \\
	}

	\Output{
		\tab, hash table of indexing functions for all \\
		$\qquad$full-dimension regions of the arrangement.
	}

	\BlankLine

	\GLOBAL{\tab $\gets \{\}$}


	\BlankLine

	\SetKwProg{Fn}{function}{}{end}%

	\BlankLine

	\tcc{Use \enum from \cref{alg:level_wise_region_enumeration} but 
	replace \succ (line 9) with:}

	\BlankLine


















		


	\Input{
		$\mathscr{L}_{\nn}$, affine functions for NN hyperplane arrangement; \\
		$\mathfrak{s}$, indexing function for a valid region. \\
	}

	\Output{
		\succList, list of \textbf{new} region index fns. adjacent \\
		$\qquad$to $\mathfrak{s}$ in the next higher/lower region poset level
	}

	\Fn{\succFwdBkwd{$\mathscr{L}_{\nn}$, $\mathfrak{s}$}}{

		\succList $\gets \{\}$

		\tcc{Flip hyperplanes to get constraints for region $R_\mathfrak{s}$ given by $\mathfrak{s}$:}

		$A \gets \left[
			\begin{smallmatrix}
				\mathfrak{s}(\mathfrak{o}^{-1}(1)) \cdot \mathfrak{o}^{-1}(1)(x) & \dots & \mathfrak{s}(\mathfrak{o}^{-1}(N)) \cdot \mathfrak{o}^{-1}(N)(x)
			\end{smallmatrix}
		\right]^\text{T}$



		\BlankLine

		\tcc{Perform forward pass from current region (note additional constraint from \cref{prop:zsub_adjacency}):}

		\succList.\appendd{ \\
			$\quad$\succ{$\mathscr{L}_{\negthickspace\nn}$, $\mathfrak{s}$, \addconstr $=\{\mathcal{T}^{\nn}_{R_\mathfrak{s}}\}$} \\
		}

		\BlankLine

		\predList $\gets \{\}$

		\BlankLine

		\tcc{Backward pass}

		\If{\solvelp{$[0\cdot1, \dots, 0\cdot n, 1]$, $A(x) + x_s \leq 0$ $~\wedge~$ $\mathcal{T}^{\nn}_{R_\mathfrak{s}}(x) + x_s \leq 0 ~\wedge~ x_s \geq 0$}.\optcost{} $> 0$ }{








			\BlankLine

			\tcc{Check only the faces intersecting $\mathcal{T}^{\negthinspace\nn}_{R_\mathfrak{s}}(x)=0$}

			\predList = \\ \succ{$\mathscr{L}_{\negthickspace\nn}$, $\mathfrak{s}$, \adjtest$=\flipset{R_\mathfrak{s}}$, \\ $\qquad\qquad$\addconstr $\negthinspace=\negthinspace\{\negthinspace\mathcal{T}^{\negthinspace\nn}_{R_\mathfrak{s}}\negthinspace\}$}

		}

		\succList.\appendd{\predList}

		\BlankLine

		\Return \succList

	}

\caption{\small$\texttt{Sub-Level Set Region Enumeration}$}
\label{alg:zero_sub_enumeration}
\end{algorithm} 
\endgroup
	\end{minipage}
\end{figure}
\DecMargin{0.5em}
 

The correctness of \cref{alg:zero_sub_enumeration} follows from the following 
theorem, the proof of which appears in \cref{sub:proof_of_main_theorem}.

\begin{theorem}
	\label{thm:main_theorem}
	Let $(\mathscr{H}_{\negthinspace\nn}, \mathscr{L}_{\negthinspace\nn})$ by a 
	hyperplane arrangement for a shallow NN, $\nn$ (see  
	\cref{prop:nn_local_linear}), and let $x_0$ be a point for which  $\nn(x_0) 
	< 0$. Assume WOLG that $x_0 \in R_b$, the base region of  
	$\mathscr{H}_{\negthinspace\nn}$.

	Then \cref{alg:zero_sub_enumeration} returns all regions of  
	$\mathscr{H}_{\negthinspace\nn}$ that intersect the connected component of  
	$\text{int}(\zsub{\nn})$ containing $x_0$.
\end{theorem}
\begin{proof}
	See \cref{sub:proof_of_main_theorem}.
\end{proof}



\ActivateWarningFilters[pdftoc]
\subsection{Checking Property (a) of \cref{prob:zero_set_prob}} 
\label{sub:checking_property_a}
\DeactivateWarningFilters[pdftoc]

For \cref{prob:zero_set_prob}\emph{(a)}, the additional, point-wise property 
that we need to check during \cref{alg:zero_sub_enumeration} is containment of 
the connected component of $\zsub{\nn}$ within $X_\partial$. Note that 
\cref{alg:zero_sub_enumeration} effectively returns a set $X_c$ such that $X_c 
= \overbar{\text{int}(X_c)}$ and $\text{bd}(X_c) \subseteq \zeq{\nn}$, so the 
main criterion of \cref{prob:zero_set_prob}\emph{(a)} is satisfied; see 
\cref{thm:main_theorem}.

However, \cref{alg:zero_sub_enumeration} can be trivially modified to identify 
only regions that intersect a separate convex polytope. This entails augmenting 
the arrangement with hyperplanes containing the polytope faces, and always 
treating those hyperplanes as ``flipped''  
(\cref{alg:level_wise_region_enumeration}, line  20). It then suffices to test 
each region returned by \cref{alg:zero_sub_enumeration} to see if it has a face 
among these unfippable polytope faces. If any region has such a face, then the 
identified component of $\text{int}(\zsub{\nn}) \not\subset X_\partial$; 
otherwise, \cref{alg:zero_sub_enumeration} verifies \emph{(a)}.


\ActivateWarningFilters[pdftoc]
\subsection{Checking Property (b) of \cref{prob:zero_set_prob}} 
\label{sub:checking_property_b}
\DeactivateWarningFilters[pdftoc]

To check \cref{prob:zero_set_prob}\emph{(b)}, we need to consider points 
outside the component $X_c \subseteq \zsub{\nnbf}$ (obtained from 
\cref{prob:zero_set_prob}\emph{(a)}) but inside a $\max$-norm ball of radius 
given in \eqref{eq:main_reach_bound}. For this, we can use 
\cref{alg:zero_sub_enumeration} on $-\nnbf(x)$, and interpret the $\max$-norm 
ball as a containing polytope (viz. hypercube) as in 
\cref{sub:checking_property_a}. For this run of 
\cref{alg:zero_sub_enumeration}, the positivity of $\nnbf$ (negativity of  
$-\nnbf$) can be checked on each region with a single LP. Thus, $\nnbf$ is 
positive on $C_{x_0}\backslash X_c$ if every region so produced passes this 
test.

It only remains to compute the radius of the $\max$-norm ball $C_{x_0}$ in  the 
first place. According to \eqref{eq:main_reach_bound} the main quantities we 
need to compute are $\sup_{x\in X_c}\lVert x - x_0 \rVert$ and $\sup_{x \in 
X_c} \lVert \nnvec(x_0) - x \rVert$; $\lVert \nnvec \rVert$, the Lipschitz 
constant of $\nnvec$ can be estimated in the trivial way or by any other 
desired means. Fortunately, both quantities involve computing the $\max$-norm 
of shifted versions of the set $X_c$. By the properties of a norm, these 
quantities can be derived directly from a coordinate-wise bounding box for 
$X_c$. Such a bounding box for $X_c$ can in turn can be computed directly from 
the regions discovered in our solution to \cref{prob:zero_set_prob}\emph{(a)}: 
simply use two LPs per dimension to compute the bounding box of each region, 
and then maintain global $\min$'s and $\max$'s of these quantities over all 
regions.





\section{Experiments}
\label{sec:experiments}

In order to validate the utility and efficiency of our algorithm, we conducted 
two types of experiments. \cref{sub:case_studies} contains case studies on two 
real-world control examples: control of an inverted pendulum in  
\cref{ssub:pendulum} and control of a steerable bicycle model 
\cref{ssub:bicycle}. This analysis is supplemented by scalability experiments 
in \cref{sub:scalability}, which evaluate the scalability of the novel 
algorithm presented in \cref{sec:hyperplane_region_enumeration}.

All experiments were run on a 2020 Macbook Pro with an Intel i7 processor and 
16Gb of RAM. In all experimental runs, the code implementing algorithms in 
\cref{sec:forward_reachability} was run directly on the host OS; by contrast,  
the code implementing algorithms from \cref{sec:hyperplane_region_enumeration} 
was run in a Docker container. All code is available in a Git  
repository\footnote{ (REDACTED FOR REVIEW) %
} %
which provides instructions to create a Docker container that can execute all 
code mentioned above (including from \cref{sec:forward_reachability}).

\subsection{Case Studies} 
\label{sub:case_studies}

Our algorithm considers autonomous system dynamics described by a ReLU NN 
vector field (\cref{prob:main_problem}) and a (shallow) ReLU NN candidate 
barrier function. Thus, in all case studies we obtain these functions via the 
following two steps: first, by training a ReLU NN to approximate the true 
open-loop system dynamics; and second, by jointly training a ReLU NN controller 
(which produces autonomous NN system in closed loop) and a shallow ReLU NN  
barrier function. Note that the closed-loop composition of a controlled NN 
vector field with a NN controller is also a NN, albeit not a shallow NN.

To obtain a controlled vector field in ReLU NN form, we start with each case 
study's actual discrete-time system dynamics (see \eqref{eq:ds_pendulum} and 
\eqref{eq:ds_bicycle}), given in general by:
\begin{equation}
\label{eq:system_nn}
	x_{t+1} = f(x_t, u_t) \in \mathbb{R}^n \quad \text{with} \quad u_t \in U \subseteq \mathbb{R}^m
\end{equation}
\noindent and define $\mathcal{X}$ as a subset of the state space that contains 
the appropriate set safe states, $X_s$, as well as other states of interest. We 
then uniformly sample $\mathcal{X} \times U$ to obtain $K = 2000$ data points 
$\{(\hat{\text{\textbf{\textsf{x}}}}_k,  
\hat{\text{\textbf{\textsf{u}}}}_k)\}_{k=1}^{K}$, which we subsequently use to 
train a ReLU NN $\nnol$ that minimizes the mean-square loss function  
$\sum_{k=1}^{K}\rVert f(\hat{\text{\textbf{\textsf{x}}}}_k, 
\hat{\text{\textbf{\textsf{u}}}}_k) - \nnol(\hat{\text{\textbf{\textsf{x}}}}_k, 
\hat{\text{\textbf{\textsf{u}}}}_k)\lVert_2^2$. In all case studies, $\nnol$ is 
a shallow NN architecture with 64 neurons in the hidden layer.

Given the NN open-loop dynamics $\nnol$, we then use the method in 
\cite{AnandZamani2023} to simultaneously train a time-invariant feedback 
controller $\nncontroller: \mathbb{R}^m \rightarrow \mathbb{R}^n$ and a 
candidate barrier function $\nnbf: \mathbb{R}^n \rightarrow \mathbb{R}$; the 
architectures of $\nncontroller$ and $\nnbf$ are described in each case study. 
From $\nncontroller$ and $\nnol$, we obtain the autonomous NN vector field as: 
$x_{t+1} = \nnvec(x_t) \triangleq \nnol(x_t,\nncontroller(x_t)).$

\subsubsection{Inverted Pendulum}
\label{ssub:pendulum}
Consider an inverted pendulum with states for angular position, $x_1$, and 
angular velocity, $x_2$, of the pendulum and a control input, $u$, providing an 
external torque on the pendulum. These are governed by discretized open-loop  
dynamics:
\begin{equation}
\label{eq:ds_pendulum}
	\left[
	\begin{smallmatrix}
		x_1(t + 1) \\
		x_2(t + 1)
	\end{smallmatrix}
	\right]
	=
	\left[
	\begin{smallmatrix}
		x_1(t) + \tau \cdot x_2(t) \\
		x_2(t) + \tau \cdot \left(\frac{g}{l} \sin(x_1(t)) + \frac{1}{m l^2} u\right)
	\end{smallmatrix}
	\right] 
\end{equation}
where $m=1\thickspace\mathtt{kg}$ and $l=1\thickspace\mathtt{m}$ represent the 
mass and length of the pendulum respectively,  
$g=9.8\thickspace\mathtt{m}/\mathtt{s}^2$ is the gravitational acceleration and 
$\tau=0.01\thickspace\mathtt{s}$ is the sampling time.

In this case study, we are interested in stabilizing the inverted pendulum 
around $(x_1,x_2) = (0,0)$ while keeping it in the safe region $X_s = [-\pi/6, 
\pi/6]^2$, so we define \(\mathcal{X} = [-\pi/4, \pi/4]^2\) and the control 
constraint $U = [-10, 10]\thickspace \mathtt{rad}/\mathtt{s}^2$. We  proceed to 
train ReLU open-loop dynamics, $\nnol$, as above; then using 
\cite{AnandZamani2023}, we train both a stabilizing controller, 
${\mathscr{N}_c}$ (shallow ReLU NN, 5 neurons), and a barrier candidate, 
$\nnbf$ (shallow ReLU NN, 20 neurons).

Our algorithm certified the green set depicted in \cref{fig:zero_level_2d}  
(Left) as a forward invariant set of states. In particular, our algorithm  
(\cref{sec:forward_reachability}) produced $X_\partial = X_s$ as a verified 
solution to \cref{prob:neg_deriv_prob} (white set in \cref{fig:zero_level_2d}  
(Left)), for which our algorithm took 1.2 seconds and produced 25 partitions. 
Our algorithm  (\cref{sec:hyperplane_region_enumeration}) then produced the 
aforementioned green set as a verified solution to \cref{prob:zero_set_prob} in 
8.27 seconds using a  Lipschitz constant estimate of 0.8 and $x_0 = (0,0)$; it 
thus certifies $\nnbf$ as a barrier for that set.

\begin{figure}[t]
	\centering
	\includegraphics[width=0.53\linewidth]{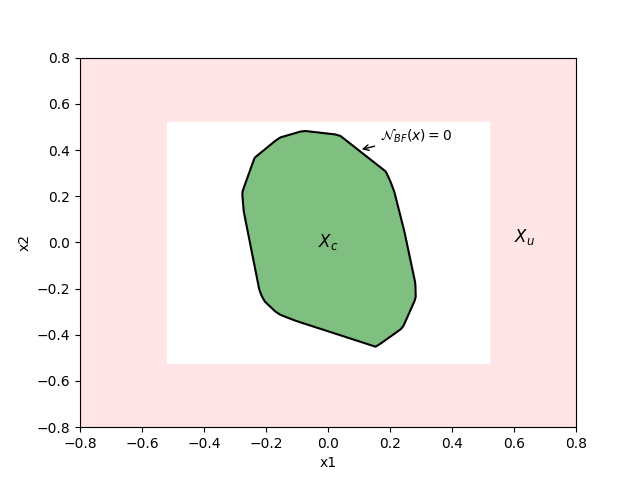} %
	\includegraphics[width=0.46\linewidth,trim={1.4in 0.51in 0.8in 0.51in},clip]{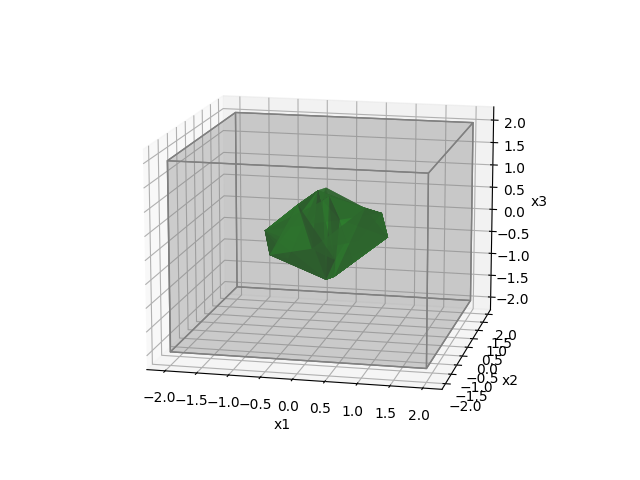}
	\caption{Certified forward invariant sets are shown in green for the inverted pendulum case study (Left) and steerable bicycle case study (Right); both sets are contained the set of safe states $X_s$, as defined in each case study ($X_s$ is shown as a white/grey box). The green sets are zero-sub-level sets of the trained $\nnbf$, and are returned by our algorithm.}
	\label{fig:zero_level_2d}
\end{figure}

\subsubsection{Steerable Bicycle}
\label{ssub:bicycle}
Consider a steerable bicycle viewed from a frame aligned with its direction of 
travel; this system has states for tilt angle of the bicycle in a plane  normal 
to its direction of travel, $x_1$, the angular velocity of that tilt, $x_2$ and 
the angle of the handlebar with respect to the body, $x_3$ and a control  input 
$u$ for the steering angle. These are governed by the open-loop dynamics:
\begin{multline} \label{eq:ds_bicycle}
\left[
\begin{smallmatrix}
	x_1(t + 1) \\
	x_2(t + 1) \\
	x_3(t + 1)
\end{smallmatrix}
\right]
=
\left[
\begin{smallmatrix}
	x_2(t) \\
	\frac{ml}{J} \left(g \sin(x_1(t)) + \frac{v^2}{b} \cos(x_1(t)) \tan(x_3(t))\right) \\
	0
\end{smallmatrix}
\right]
\\ +
\left[
\begin{smallmatrix}
	0 \\
	\frac{amlv}{Jb} \cos(x_1(t)) \cos^2(x_3(t)) \\
	1
\end{smallmatrix}
\right]
u;
\end{multline}
where $m=20\thickspace\texttt{kg}$ is the bicycle's mass,  
$l=1\thickspace\texttt{m}$ its height, $b=1\thickspace\texttt{m}$ its wheel 
base, $J = \tfrac{ml}{3}$ its moment of inertia,  
$v=10\thickspace\texttt{m}/\texttt{s}$ its linear velocity,  
$g=9.8\thickspace\texttt{m}/\texttt{s}^2$ is the acceleration of gravity, and 
$a = 0.5$. 

In this case study we are seek to stabilize the bicycle in its vertical 
position while keeping it in the safe region $X_s = [-2, 2]^3$, so we define  
$\mathcal{X} = [-2.2, 2.2]^3$ and the control constraint is $U = [-10,10]$. We  
train ReLU open-loop dynamics, $\nnol$, as above; then using 
\cite{AnandZamani2023}, we train a stabilizing controller, ${\mathscr{N}_c}$ 
(shallow ReLU NN, 5 neurons), and a barrier candidate, $\nnbf$ (shallow ReLU 
NN, 10 neurons).

Our algorithm certified the green set depicted in \cref{fig:zero_level_2d}  
(Right) as a forward invariant set of states. In particular, our algorithm 
(\cref{sec:forward_reachability}) produced $X_\partial = X_s$ as a verified 
solution to \cref{prob:neg_deriv_prob} (grey set in \cref{fig:zero_level_2d}  
(Right)), for which our algorithm took 9.52 seconds and produced  125 
partitions. Our algorithm  (\cref{sec:hyperplane_region_enumeration}) then 
produced the aforementioned green set as a verified solution to 
\cref{prob:zero_set_prob} in 8.76 seconds using a  Lipschitz constant estimate 
of 0.78 and $x_0 = (0,0,0)$; it thus certifies $\nnbf$ as a barrier for that 
set.





\subsection{Scalability Analysis}
\label{sub:scalability}



\begin{figure}[t]
	\centering
	\includegraphics[width=0.65\linewidth,trim={1.9in 0.2in 10.4in 0.7in},clip]{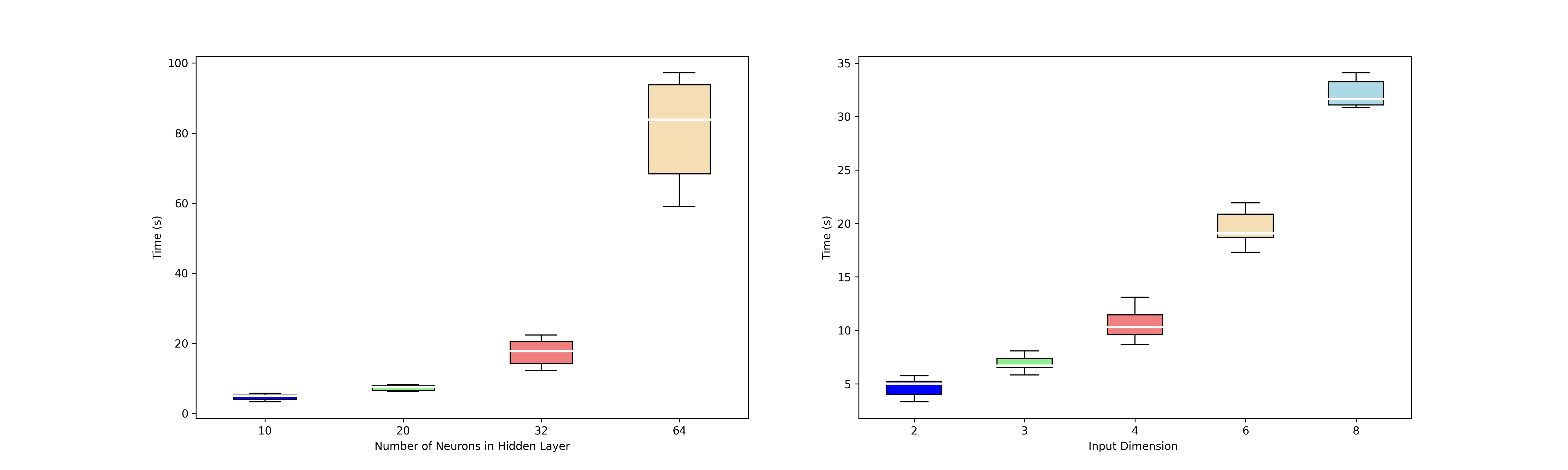}
	\caption{Zero-sub-level Set Verification Time}
	\label{fig:levelset_neurons}
\end{figure}

\begin{figure}[t]
	\centering
	\includegraphics[width=0.65\linewidth,trim={10.4in 0.2in 1.9in 0.7in},clip]{levelset_time.png}
	\caption{Zero-sub-level Set Verification Time}
	\label{fig:levelset_dimension}
\end{figure}

Our algorithm for \cref{prob:neg_deriv_prob}  (see 
\cref{sec:forward_reachability}) is based on an existing tool (viz. CROWN 
\cite{ZhangEfficientNeuralNetwork2018}), so we focus our scalability study on 
our novel algorithm for solving \cref{prob:zero_set_prob}  (see 
\cref{sec:hyperplane_region_enumeration}), i.e. certifying zero-level-sets  for 
shallow NN barrier functions. We study scaling both in terms of the candidate 
barrier NN's input dimension (for a fixed number of neurons) and in the number 
of neurons in the candidate barrier NN (for a fixed input  dimension).

To conduct this experiment, we trained a number of ``synthetic'' candidate 
barrier function NNs with varying combinations of input dimension and number of 
hidden-layer neurons. We refer to these as synthetic barriers, since they were 
created without reference to any particular dynamics or control problem; i.e. 
they were all trained on datasets of $K=500$, 
$\{(\hat{\mathbf{x}}_k,\hat{\mathbf{y}}_k)\}_{k=1}^{K}$ such that  
$\hat{\mathbf{x}}_k \in [-1,1]^d \implies \hat{\mathbf{y}}_k = -1$ and 
$\hat{\mathbf{x}}_k \not\in [-1,1]^d \implies \hat{\mathbf{y}}_k = 1$. This 
nominally incentivizes the hypercube $[-1,1]^d$ to be contained in their 
zero-sub-level set. The rest of the inputs required for our algorithm were 
generated as follows -- see \eqref{eq:main_reach_bound} and recall there is no 
referent closed-loop dynamics: the Lipschitz estimate was chosen uniformly from 
$[0, 1.2]$; the initial point was $x_0 = (0,  \dots, 0)$; the ``next state'' 
from $x_0$ was generated via a coordinate-wise offset from $x_0$ drawn 
uniformly from $[0,0.1]$; and a ``synthetic'' set $X_\partial$ was generated as 
a single $[-10,10]^d$ hyperrectangle for $d>3$ and four manually specified  
hyperrectangles for $d \leq 3$.

\cref{fig:levelset_neurons} summarizes our neuron scaling experiment with a 
box-and-whisker plot of NN barrier candidate size (in neurons) vs. execution 
time (in seconds) for our zero-sub-level set algorithm. All NN barrier 
candidates are synthetic NN barrier candidates as described above with a common 
input dimension of 2. This experiment confirms that our algorithm and its 
implementation scale similarly to hyperplane region enumeration, i.e. $O(N^d)$ 
where $N$ is the number of neurons; for example, the median runtime for $N=64$ 
is roughly $4$ time the median runtime for $N=32$ neurons.

\cref{fig:levelset_dimension} summarizes our dimension scaling experiment with 
a box-and-whisker plot of NN barrier candidate input dimension vs. execution 
time (in seconds) for our zero-sub-level set algorithm. All NN barrier 
candidates are synthetic NN barrier candidates as described above with a 10 
hidden layer neurons. This experiment confirms our algorithm and its 
implementation scale as hyperplane region enumeration, viz. exponentially in 
input dimension. 



 %


\ActivateWarningFilters[pdftoc]
\section{Appendix: Proof of \cref{thm:main_theorem}} 
\label{sub:proof_of_main_theorem}
\DeactivateWarningFilters[pdftoc]


To facilitate the proof, we introduce the following definitions.

\begin{definition}[Fold-Back Face]
	\label{def:fold_back_face}
	Let $(\mathscr{H}_{\nn}, \mathscr{L}_{\nn})$ be a hyperplane arrangement 
	based on a shallow NN, $\nn$. Also, let $R$ be a region of this arrangement 
	(with indexing function $\mathfrak{s}$), and let $F$ be a  
	(full-dimensional) face of $R$.

	Then $F$ is a \textbf{fold-back face} of $R$ if $\exists \ell \in  
	\fliph{R}$ such that
	\begin{equation}
	\label{eq:fold_back_face}
		F \subset H^0_{\ell} \;\wedge\;
		H^0_{\mathcal{T}^{\negthinspace\nn}_{R}} \cap \bar{F} \neq \emptyset \;\wedge\;
		H^{-1}_{\mathcal{T}^{\negthinspace\nn}_{R}} \cap F \neq \emptyset.
	\end{equation}
	Note the closure of $F$ in the second condition.
\end{definition}

\begin{definition}[Fold-Back Region]
	\label{def:fold_back_region}
	Let $(\negthinspace\mathscr{H}_{\nn},\negthinspace  
	\mathscr{L}_{\nn}\negthinspace)$ be a hyperplane arrangement for the 
	shallow NN, $\negthinspace\nn\negthickspace$. A region of this arrangement 
	is a \textbf{fold-back region} if it has at least one fold-back face.
\end{definition}

\begin{remark}
	``Fold-back'' is meant to evoke the case illustrated in 
	\cref{fig:backward_pass}: e.g. $R_4$ is a fold-back region of $R_3$, 
	because the boundary of $\zsub{\nn}$ is ``folded back'' across an already 
	flipped hyperplane in $R_3$.
\end{remark}
\noindent Now we proceed with the proof of \cref{thm:main_theorem}.
\begin{proof}(\cref{thm:main_theorem}.)
We need to show that the hash table created by \cref{alg:zero_sub_enumeration} 
contains all of the regions of $(\mathscr{H}_{\nn}, \mathscr{L}_{\nn})$ that 
intersect the connected component $C \subseteq \text{int}(\zsub{\nn})$  where 
$x_0 \in C $.

To do this, first observe that \cref{alg:zero_sub_enumeration} adds regions to 
the table by exactly two means: the ``forward'' pass, which calls  
\texttt{FindSucces}-\texttt{sors} on $\unflipset{R}$ (see line 9); and the 
``backward'' pass, which calls \texttt{FindSuccessors} on $\flipset{R}$ (see 
line 16). Moreover, \cref{alg:zero_sub_enumeration} performs at most one of 
each \texttt{FindSuc}-\texttt{cessors} call per region, and the returned table 
is the union of all regions discovered by these calls. Thus, the output of 
\cref{alg:zero_sub_enumeration} is equivalent to repeating the following 
two-step sequence until the table no longer changes: iteratively performing  
forward passes of \texttt{FindSuccessors} until the table no longer changes; 
followed by iteratively performing backward passes of \texttt{FindSuccessors} 
until the table no longer changes.

Furthermore, to facilitate this proof, we assume backward passes only add 
regions connected via fold-back faces. Since this algorithmic modification 
creates a region table that is a subset of that created by 
\cref{alg:zero_sub_enumeration}, it suffices to prove the claim in this case.

With this in mind, we define the following notation.
\begin{align}
	\mathcal{f}_k &: R \subset \mathscr{R} \mapsto \label{eq:fwd_pass_notation} \\
	&\scriptstyle\bigcup_{R^\prime \in \mathcal{f}_{k-1}(R)} \negthickspace
	\big\{ R^{\prime\prime} \cap H^{-1}_{\mathcal{T}^{\nn}_{R^{\prime\prime}}} ~ | ~ R^{\prime\prime} \in 
	\scriptstyle\mathtt{FindSuccessors}(\mathscr{L}, R^\prime, \mathsf{testHypers}=\flipset{R^\prime})
	\big\}
	\notag \\
	\mathcal{f}_0 &: R \subset \mathscr{R} \mapsto \{ R \}
	\label{eq:fwd_pass_notation_base} \\
	\mathcal{f} &: R \mapsto \cup_{k=0}^{\infty} \mathcal{f}_k(R)
\end{align}
We likewise define $\mathcal{b}_k$, $\mathcal{b}_0$ and $\mathcal{b}$ based on 
backward passes, i.e. using $\unflipset{R}$ in \eqref{eq:fwd_pass_notation}. In 
this way, we can describe the overall output of \cref{alg:zero_sub_enumeration} 
(for the purposes of this proof) using the notation:
\begin{align}
	\mathcal{o}_{k\hphantom{-}} &: \begin{cases}
		\mathcal{f}(\mathcal{o}_{k-1}) \backslash \cup_{\nu = 1}^{k-2} \mathcal{o}_\nu & \text{if } k \in \mathbb{N} \text{ is odd} \\
		\mathcal{b}(\mathcal{o}_{k-1}) \backslash \cup_{\nu = 1}^{k-2} \mathcal{o}_\nu & \text{if } k \in \mathbb{N} \text{ is even}
	\end{cases} 
	\label{eq:both_passes} \\
	\mathcal{o}_{0\hphantom{-}} &\triangleq \{R_b\}
	\label{eq:both_pass_base} \\
	\mathcal{o}_{-1} &\triangleq \emptyset.
\end{align}
where $R_b$ is the base region of  
$(\mathscr{H}_{\negthinspace\nn},\mathscr{L}_{\negthinspace\nn})$ as usual. 
Thus, the union of the table output by (the restricted version of) 
\cref{alg:zero_sub_enumeration} is:
\begin{equation}
\label{eq:o_definition}
	\mathcal{o} \triangleq \cup_{k = 1}^L \bar{\mathcal{o}}_k
\end{equation}
where $L$ is the first integer such that $\mathcal{o}_L = \mathcal{o}_{L-1} = 
\emptyset$.

Now let $p : [0,1] \rightarrow C  \subseteq \text{int}(\zsub{\nn})$ be a 
continuous curve between two points $p(0), p(1) \in C$. To prove the claim, it 
suffices to show that $p(0), p(1) \in \text{int}(\mathcal{o})$, and hence the 
connected component $C \subseteq \text{int}(\mathcal{o})$ because  
$\text{int}(\mathcal{o})$ is connected by construction: that is, every point in 
$C$ is connected through $\text{int}(\mathcal{o})$. The reverse direction is 
true by construction. We can also assume without loss of generality that $p(0) 
= x_0$; also let $x_f := p(1)$ for any $p(1)$ as above.

We now proceed by contradiction: that is, we suppose that $x_f \in C$ but $x_f  
\not\in \text{int}(\mathcal{o})$. The case when $x_f \in  
\text{bd}(\mathcal{o})$ is trivial, so we assume that $x_f \in  \mathcal{o}^C$, 
and thus $x_f$ is in the open set $D := C \cap \mathcal{o}^C$. Let 
$\mathfrak{d}$ be the set of hyperplane regions that intersect $D$.

Since $\mathcal{o}$ is the closure of a finite number of (open) polytopes, we 
conclude $\text{bd}(D)$ consists of faces of regions in $\mathcal{o}$ and/or 
zero crossings, i.e. $H^0_{\mathcal{T}^{\nn}_R} \cap R$ for regions $R  : R 
\cap D \neq \emptyset$. Note that $\text{bd}(D)$ must have at least one face, 
$F$, that is also face of a region $R \in \mathfrak{d}$; for if not,  it 
contradicts $x_f \in D \subset \text{int}(C)$. Those faces $F \cap \text{bd}(D) 
\neq \emptyset$ which are entirely zero crossings are of no interest to 
\cref{alg:zero_sub_enumeration}.

Now let $F$ be any such face that is a face of $R \subset \mathfrak{o}$ as well 
as $F \cap \text{bd}(D) \neq \emptyset$. We claim that $F$ is contained in a 
hyperplane associated to $\ell_F$ that is flipped for the region $R$ and 
\emph{unflipped} for $D$; for if it were \emph{unflipped} in $R$, then 
\cref{alg:zero_sub_enumeration} would add the region adjacent to $R$ through 
$F$ via a forward pass. Neither can $F$ (via $\ell_F$) correspond to a flipped 
hyperplane for $R$ and also be a fold-back face of $R$: for if $\ell_F$ were as 
such, then a backward pass would add another region to $\mathcal{o}$, which 
contradicts the definition of $\mathfrak{d}$.

At this point, we simply observe that not all shared faces between $D$ and  
$\mathcal{o}$ can correspond to flipped hyperplanes for their adjacent regions 
in $\mathfrak{o}$ and simultaneously not be fold-back faces. For if this were 
so, then the line connecting $x_f$ to  $x_0$ would necessarily go through an 
unflipped hyperplane (w.r.t. $D$), and this is clearly impossible. Thus, $D$ 
must have at least one face in common with a region in $\mathcal{o}$ that is 
either unflipped in $\mathcal{o}$ or else a fold-back face of a region in  
$\mathcal{o}$. In either case, we have a contradiction with the fact that $D$ 
contains regions undiscovered by \cref{alg:zero_sub_enumeration}.
\end{proof}



 %

\bibliographystyle{plain} %
\bibliography{bibliography}


\end{document}